\documentclass{new_tlp}
\usepackage{mathptmx}

\usepackage{times,amsfonts}
\usepackage{helvet}
\usepackage{courier}
\usepackage{latexsym}

\usepackage[ruled]{algorithm2e}
\usepackage{url}

\usepackage{graphics,graphicx}
\usepackage{epsf}
\usepackage{rotating}
\usepackage{epsfig}
\usepackage{color}
\usepackage{enumerate}

\usepackage{amsmath,amssymb}
\usepackage{epsfig}
\usepackage{booktabs}

\usepackage{comment}
\long\def\comment#1{}

\newcommand{\boldK}{{\textit {\bf  K}\hspace{.01in}}}
\newcommand{\boldnot}{{\textit {\bf  not}\,}}
\newcommand{\bfnot}{\textit {\bf  {not}}}
\newcommand{\bfK}{{\textit {\bf  K}}}

\newcommand{\boldP}{\textit {\bf  P}}
\newcommand{\boldN}{\textit {\bf  N}}

\newcommand{\MKNF}{{\textit {MKNF}}}
\newcommand{\bft}{{\textit {\bf t}}}
\newcommand{\bff}{{\textit {\bf f}}}

\newcommand{\head}{\textit{hd}}
\newcommand{\body}{\textit{bd}}
\newcommand{\wholebody}{\textit{body}}

\newcommand{\lfp}{\textit{lfp}}
\newcommand{\Approx}{\textit{Appx}}

\newcommand{\lub}{\textit{lub}}
\newcommand{\glb}{\textit{glb}}

\newcommand{\K}{\cal K}
\newcommand{\M}{{\cal M}}
\newcommand{\N}{{\cal N}}
\newcommand{\KA}{\sf KA}
\newcommand{\OB}{\sf OB}

\renewcommand{\P}{\cal P}
\renewcommand{\Pi}{\varPi}
\renewcommand{\Gamma}{\varGamma}

\newcommand{\cK}{\mathcal K}
\newcommand{\cO}{\mathcal O}
\newcommand{\cP}{\mathcal P}
\newcommand{\cL}{\mathcal L}

\newcommand{\imply}{\supset}

\usepackage{lineno,hyperref}
\modulolinenumbers[5]

\newtheorem{example}{Example}
\newtheorem{definition}{Definition}
\newtheorem{proposition}{Proposition}
\newtheorem{theorem}{Theorem}

\newcommand\bcmdtab{\noindent\bgroup\tabcolsep=0pt%
  \begin{tabular}{@{}p{10pc}@{}p{20pc}@{}}}
\newcommand\ecmdtab{\end{tabular}\egroup}

  \title[Theory and Practice of Logic Programming]
        {Alternating Fixpoint Operator for Hybrid MKNF Knowledge Bases as an Approximator of AFT}

  \author[F. Liu and J-H. You]
         {Fangfang Liu$^1$ and Jia-Huai You$^2$\\
         $^1$School of Computer Engineering and Science, Shanghai University, Shanghai, China\\
         $^2$Department of Computing Science, University of Alberta, Edmonton, Canada\\
 \email{ffliu@shu.edu.cn, jyou@ualberta.ca}
}

\jdate{March 2003}
\pubyear{2003}
\pagerange{\pageref{firstpage}--\pageref{lastpage}}
\doi{S1471068401001193}

\newtheorem{lemma}{Lemma}[section]

\begin{document}

\clearpage
\pagenumbering{arabic}

\comment{
\setcounter{page}{-0}

\noindent
{\bf Comments on Revision:}

Besides correcting the typos indicated in guest editor's note,  we have slightly polished the following parts of the paper. 

\begin{itemize}
\item Footnote 5 is slightly re-worded. There was a grammar error in the original sentence. 

\item On page 14, ``...,  we are then talking about a two-valued MKNF model"  $\Rightarrow$ " ..., then $(M,N)$ is equivalent to a two-valued MKNF model."   The latter is more formal. 

\item On page 20, ``..., but its hd(r) does not; ..." $\Rightarrow$ ``..., but its hd(r) evaluates to f or u".  The latter is explicit, instead of saying ``does not (evaluate to t)". 

\item On page 21, starting from the 4th line, we added more details about the case when M = N: ``If M = N, one can verify ...."

\item On page 23, in then first paragraph of proof of Theorem 6,  the word ``isolated" is removed as it does not say anything informative, and the sentence is slightly revised. 

\item The last sentence of Section 7 is slightly revised. 

\end{itemize}
}

\comment{
\noindent
{\bf Comments on Revision:}

We'd like to thank the reviewers for their detailed comments and suggestions, which guided us through this revision.  The main revision is to separate the original AFT (new section 2.1) from our generalization (section 3).
Addtional responses are attached below. All numbers (examples, sections,  et al.) refer to the revised version presented here if not said otherwise. 

\medskip
\noindent
{\bf Reviewer 1}
\begin{itemize}
\item 
{\color{blue} ``In addition,
more intuitive/motivating examples should be added, in particular in
the first half of the paper, to make the paper more accessible for
non-experts in AFT and/or hybrid MKNF."}

Response: Example 1 on page 6 is added to explain various technical definitions of AFT. Example 4 just before Section 4.2 is added to illustrate the definition of 3-valued MKNF models. 

\medskip
\item {\color{blue} ``The paper starts with two paragraphs which already mention many
details of AFT, which are formally introduced only later in the paper.
For instance, the projection functions for pairs and monotone
operators are only defined in Section 2. This makes these paragraphs
difficult to understand for a reader not familiar with AFT. Moreover,
it may be better to start with the general introduction of AFT and to
go into the specifics afterwards in the introduction.''}

R: The Introduction is revised to be general, and the details of AFT are removed. 

\medskip
\item {\color{blue} ``The first two sentences of the paragraph on page 2 starting `A major
advantage...' are not very clear without more context and/or
background on AFT."}

R: Revised by mentioning ``algebraic structure". Also changed the word `advantage' to `feature' to sound more self-contained. 
In this new version, it is given at the start of the bottom paragraph of page 1.

\medskip
\item {\color{blue}
``The citation format should be fixed and aligned throughout the paper.
In many places, the author names appear twice, ..."}

R: Fixed throughout.

\medskip
\item {\color{blue}
``In Section 2, for instance the concepts of bilattice, chain,
A-contracting and A-prudent could be illustrated by examples to
provide more intuition."}

R:
The familiar bilattice in logic programming is now mentioned just before Section 2.1.
In this new version, we added Section 2.1 on previous development of AFT, where motivations and explanations of various notations are given in some details and illustrated in Example 1.

\medskip
\item {\color{blue} ``It is defined what it means for a pair to be A-contracting and
A-prudent, but no explanation of these concepts is given; and the
notation $\lfp(A(\cdot,v)_1)$ is mentioned at the beginning of the
introduction but should be defined in the preliminaries section as
well."}

R: Now it is given in Section 2.1.

\medskip
\item 
{\color{blue}
``In Section 3, the background on AFT is interleaved with the new
generalization of AFT. It would be more clear to introduce the
necessary preliminaries on AFT in the preliminaries section, and to
only present the generalization in Section 3. In this regard, maybe it
would be good to introduce a modified term and notation for
approximators according to the relaxed definition, to distinguish them
from the previous approximators."}

R: Thanks for this suggestion, which led to a major improvement of paper's organization, namely, a separation of the previous development of AFT (Section 2.1) and our generalization (Section 3).

\medskip
\item
{\color{blue} ``The sentence in the first paragraph of Section 3 starting `This
generalization is needed...' is not very clear without the subsequent
context on approximators and hybrid MKNF."}

R: Now the sentence is followed by a short, general explanation. This can be found in the first paragraph of Section 3 (page 7).

\medskip
\item {\color{blue} ``The projection operators used in equation (1) should be introduced
formally in addition to the projection functions."} 

The definition is now given along with equation (1).

\medskip
\item {\color{blue}
``In the sentence `... we formulate for hybrid MKNF knowledge bases are
(essentially) strong approximators' on page 8, what precisely is meant
by 'essentially'?"}

R:  Good point. We added a footnote (footnote 9 on page 10). The formal claim requires a mild condition which does not prohibit the preservation of consistent stable fixpoints that correspond to intended models. 

\medskip
\item
{\color{blue}
``Theorem 3 should be introduced by explaining why having
chain-completeness is important."}

R: In the revised version, this theorem becomes Theorem 4. As we now separate the past development from our generalization, we get a chance to explain why in the original AFT chain-completeness is important (Section 2.1). Also, we added a short paragraph commenting on why chain-completeness is relevant in the generalized AFT (the last paragraph of Section 3).

\medskip
\item {\color{blue}
``In Section 4, page 10, the sentence starting `Then we extend the
language...' does not refer to the definition of the satisfaction
relation behind the colon. This part should be moved behind the
indented definition. Or the definition could also be moved to a figure
like Table 1, to align the formatting."}

R: The “names” should come first because the definition of satisfaction makes a reference to it. Re-organized

\medskip
\item {\color{blue} ``The term DL-safety is mentioned on page 12, but it is not explained
what it means."}

R: The definition is now provided (this part of material is now on page 15).  DL-safety ensures that a first-order rule base is semantically equivalent to a finite ground rule base; hence decidability is guaranteed.

\medskip
\item
{\color{blue}
``It would be good to see a comment on why the notion of partition is
generalized from consistent pairs to all pairs on pages 12 and 14."}

R: A comment is added



\medskip
\item {\color{blue} 
``The proofs are sometimes difficult to read when many steps are
condensed into a long sentence; and starting embedded sentences with a
formula should be avoided. For instance, the sentence in the proof for
Theorem 3 starting `To show that the stable revision operator is
well-defined, ...' is difficult to parse; likewise the sentences in
the proof for Theorem 5 starting 'Next, by the definition of operators...'  and `For each $r \in \cal{P}$ with $head(r)={\bf K} a$ ...'. "}

R:   The proofs have been polished. 

\medskip
\item {\color{blue} ``How is the work related to the approach in [1], where inconsistencies
in hybrid MKNF under a paraconsistent well-founded semantics are also
regarded in an alternating fixpoint construction? If there is a
relevant relation, this should be discussed in Section 7. In Section
7, it is also stated that "in (Liu and You 2017), the construction was
related to a notion called stable partition"; it should be explained
how this relates to the current approach."}

{\color{blue} ``[1] Tobias Kaminski, Matthias Knorr, João Leite: Efficient Paraconsistent Reasoning with Ontologies and Rules. IJCAI 2015: 3098-3105."}

R: The relation with (Liu and You 2017) is further explained. Thanks for pointing out this paper. As commented towards the end of Section 7, an interesting future work on paraconsistent reasoning for hybrid MKNF KBs appears to be just around the corner! 

\medskip
\item {\color{blue}  Minor Issues/Typos}

R: All corrected/revised. But the notation $\glb(S)$ and $\lub(S)$ has been used in the literature without a reference to $L$, so we just follow, assuming $L$ is clear from the context.

\end{itemize}

\bigskip
\noindent
{\bf Reviewer 2}: Minor issues are all addressed and typos fixed. Here are some comments on the revision. 

\begin{itemize}
\item {\color{blue}    ``The introduction should contain a paragraph in which the main
differences of the present paper compared to the previous conference
papers are summarized."
}

Response: The main difference is to develop the generalized AFT in full detail ad filling in technical details.

\medskip
\item {\color{blue} ``In Theorem 3, $L^{rp}$ seems not defined. ..."}

R: As suggested by another reviewer, we now separate the original AFT (Section 2.1) from our generalization (Section 3). 
Now $L^{rp}$ is introduced with motivation in Section 2.1 for $L^c$ and reintroduced, without confusion, in Section 3 for $L^2$. 

\medskip
\item {\color{blue} ``Citations: Duplication in citations, like Knorr et al. (Knorr et al.,
2011) should be removed and replaced by (Knorr et al., 2011). ..."}

R: Fixed. 

\medskip
\item {\color{blue}       ``The conclusion mainly focuses on future work for AFT; perhaps the
authors could share their thought on future work on MKNF as well?"} 

R: See the added two paragraph towards the end of Section 7.

\medskip
\item {\color{blue} `` p5: 
this generalization is motivated by operators arising in
knowledge representation that are symmetric." It would be good to give
examples here."}

R:  A footnote (footnote 3) is added. 

\medskip
\item {\color{blue} ``Let us call the existence of a gap between the two pairs of least
fixpoints discussed above" it is not entirely clear ..."}

R: Now it is made explicit. 

\medskip
\item {\color{blue} ``p 9: you write $(\lub(C_1), \glb(C_2)) $ ...

It is then not clear why on line 7 of the proof $a \leq u_0$ should hold.
"} 

R: It is caused by a typo. Once the typo is fixed, the problem is no longer there. 

\medskip
\item {\color{blue} ``p 10: ...
about the use of names: the remark in footnote 6 and the subsequent
paragraph "An MKNF interpretation M ..." are confusing, at least, if
not misleading.
  }

R: The footnote is removed.  In the revised manuscript,  first we define the satisfaction relation with “names”. Then, standard name assumption ensures countability, just like what the reviewer commented. 

\medskip
\item {\color{blue} ``To call a first-order formula satisfiable if a model exists, but not
introduce this terminology for an MKNF formula sounds a bit strange.
Note that "consistent" and "satisfiable" are used a bit
confusingly. ..."}

R: Removed the definitions of consistent, inconsistent, satisfiable here. 
In this paper, we only use “satisfiable” for first-order formulas in $\OB_{\cO, T}$, which has no confusion. For an MKNF formula, we only talk about its MKNF models or 3-valued MKNF models. So, there is no need to stress whether an MKNF formula possesses an MKNF model (resp. a three-valued MKNF model) or not.

\medskip
\item {\color{blue} ``Can you give an intuition why $N=M$ should require $N'=M'$?" }

R: First, we put the definition of 3-valued MKNF model in a definition environment (to be consistent with how 2-valued MKNF models are introduced).
If $M = N$, we are then talking about a two-valued MKNF model. The requirement 
$M' = N'$ reduces the definition to one for 
two-valued MKNF models as given in Definition \ref{2m}, which enables Knorr et al. \shortcite{KnorrAH11} to show 
that 
an MKNF interpretation pair $(M,M)$ that is a three-valued MKNF
model of $\varphi$ corresponds to a two-valued MKNF model $M$ of $\varphi$ as defined in \cite{Motik:JACM:2010}.
Example 4 is added to illustrate the definition of 3-valued MKNF model.

\medskip
\item {\color{blue} 
``The notion of DL-safety should be defined and illustrated."}

R:  Now it is defined and commented on why it is needed (in the revised version, this can be found in the top part of page 15).

\medskip
\item {\color{blue} ``The notion of `partition' $(T,P)$ is in contrast to the usual notion,
where a set is split in (nonempty) disjoint subsets. I have no
concrete suggestion, but the current one is confusing with
mathematical convention."}

R: This is the notation used in AFT. It is convenient to use the same notation for both AFT and semantics. 
\medskip
\item {\color{blue} ``
Example 4: Couldn't be the firt-order part be $a \supset \neg b$? As
said, these are simple ontology axioms (and Horn), the usefulness of
the notion of well-founded models seems questionable."}

R: Yes, or $\neg (a \wedge b)$. Perhaps an atom like $h$ appearing in $\cO$ but not in $\P$ can be a bit illustrating.
The notion of well-founded model defined by Knorr et al.  may  not really ``well-founded" in the conventional sense. 
This is perhaps reasonable since in hybrid systems inconsistencies can naturally arise, which can deny a least fixpoint to be 
an (indended) model. 

\medskip
\item {\color{blue} `` p 17: `(where either $P'=P$ ...)' should be deleted;}
{\color{blue}
'one can check that $(M',N') \not\models \pi (\mathcal{K})$' please
be more explicit here (add details)."}

R: Removed. For the latter remark, the proof of Theorem 5 (old Theorem 4) is polished, now with full details. 

\medskip
\item {\color{blue} ``p 18: "for A, it can only be discovered" what does it refer to?"}

R: The sentence is rewritten to make it straightforward. The text is now in the last part of the second paragraph of Section 6.

\medskip
\item {\color{blue} ``
`p20: ... and therefore $ T \subseteq \Phi(T)$' please detail".

`` `the fixpoint construction confirms that' please reformulate `confirms' "
}

R: The remark seems to refers to ``and therefore $T \subseteq \theta(T)$" in the proof of (old) Theorem 5. The details are now provided. In this new version, it is Theorem 6 (now on page 23, the mentioned reasoning step is in the second paragraph of the proof of Theorem 6).
``Confirms" is reformulated.

\medskip
\item {\color{blue} ``p 21: 'to capture all transitions for hybrid MKNF knowledge bases' 
it is not clear what kind of 'transitions' you mean here."}

R: This is about mapping one state (in terms of a partial partition) to another state. Re-worded.

\medskip
\item {\color{blue} ``Chiaki Sakama, Katsumi Inoue: Paraconsistent Stable Semantics for
Extended Disjunctive Programs. J. Log. Comput. 5(3): 265-285 (1995)

\medskip
Giovanni Amendola, Thomas Eiter, Michael Fink, Nicola Leone, João
Moura: Semi-equilibrium models for paracoherent answer set
programs. Artif. Intell. 234: 219-271 (2016)"}

R:  Added a paragraph (second to last) to Section 7. AFT only provides a framework where semantics are defined in terms of approximators. Whether such an approximator exists for these semantics is open. We are not very optimistic since it may be hard to express special tricks in defining such a semantics by a single approximator.  But in the same paragraph,
a similar question raised by another reviewer on paraconsistent reasoning for hybrid KBs can be interesting (and promising).

\end{itemize}

\bigskip
\noindent
{\bf Reviewer 3}:  Typos and corrections are all fixed. 

\begin{itemize}
\item {\color{blue} ``p. 4: 'Vennekens et al. (Vennekens et al. 2006)' $\rightarrow$ Vennekens et al. (2006)
Change other citations in a similar manner."}

Response: All fixed. 

\medskip
\item {\color{blue} `` p. 6: 'This paper is a substantial revision and extension of a
preliminary report of the work that appeared in (Liu and You 2019).'
$\rightarrow$  Please state exactly what the extensions."
}

R: Stated now. ``Substantial extension" is probably an overstatement. Yes, the revised paper is much longer, but 
the main extension is filling out technical details of the work through a reorganization. 
\end{itemize}
}

\label{firstpage}

\maketitle

  \begin{abstract}
    Approximation fixpoint theory (AFT) provides an algebraic framework for the study of fixpoints of operators on bilattices and has found its applications in characterizing semantics for various classes of logic programs and nonmonotonic languages. In this paper, we show one more application of this kind: the alternating fixpoint operator by Knorr et al. for the study of the well-founded semantics for hybrid MKNF knowledge bases is in fact an approximator of AFT in disguise, which, thanks to the abstraction power
of AFT, characterizes not only the well-founded semantics but also two-valued as well as three-valued semantics for hybrid MKNF knowledge bases. Furthermore, we show an improved approximator for these knowledge bases, of which the least stable fixpoint is information richer than the one formulated from Knorr et al.'s construction. This leads to an improved computation for the well-founded semantics.  This work is built on an extension of AFT 
that supports consistent as well as inconsistent pairs in the induced product bilattice, to deal with inconsistencies that arise in the context of hybrid MKNF knowledge bases. This part of the work can be considered generalizing the original AFT from symmetric approximators to arbitrary approximators. 
  \end{abstract}

  \begin{keywords}
    Approximation Fixpoint Theory, Hybrid MKNF Knowledge Bases, Logic Programs, Answer Set Semantics, Description  Logics,  Inconsistencies
  \end{keywords}

\comment {\tableofcontents}

\section{Introduction}

AFT 
 is a framework for the study of semantics of nonmonotonic logics based on operators and their fixpoints \cite{DeneckerMT04}. Under this theory, the semantics of a logic theory is defined or characterized in terms of respective stable fixpoints constructed by employing an {\em approximator} on a (product) bilattice. 
The least stable fixpoint of such an approximator is called the well-founded fixpoint, which serves as the basis for a well-founded semantics, and the stable  fixpoints that are total characterize a stable semantics, while partial  stable fixpoints give rise to a partial stable semantics. 
The approach is highly general as it only depends on mild conditions on
 approximators, and highly abstract as well since the semantics is given in terms of an algebraic structure. 
As different approximators may represent different intuitions, AFT provides a powerful framework to treat semantics uniformly and allows to explore alternative semantics by different approximators. 

Due to the underlying algebraic structure, a main feature of AFT is that  we can understand some general properties of a semantics without referring to a concrete approximator. For example,  the well-founded fixpoint approximates all other fixpoints, and mathematically, this property holds for all approximators. An implication of this property is that it provides the bases for building constraint propagators for solvers;  for logic programs for example, it guarantees that the true and false atoms in the well-founded fixpoint remain to hold in all stable fixpoints, and as such, the computation for the well-founded fixpoint can be adopted as constraint propagation for the computation of stable fixpoints. 
For example, this lattice structure  of stable fixpoints has provided key technical insights in building a DPLL-based solver for normal hybrid MKNF knowledge bases \cite{Ji17}, while previously the only known computational method was based on guess-and-verify  \cite{Motik:JACM:2010}.

AFT has been applied to default logic as well as autoepistemic logic, and the study has shown how the fixpoint theory induces the main and sometimes new semantics and leads to new insights in these logics \cite{DeneckerMT03}, including the well-founded semantics for 
autoepistemic logic \cite{DBLP:journals/ngc/0001JCJBD16}.
AFT has been adopted in the study of the semantics of logic programs with aggregates \cite{PelovDB07} and disjunctive HEX programs \cite{AnticEF13}.
Vennekens et al. \shortcite{Vennekens:TOCL2006} used AFT in a
modularity study for a number of nonmonotonic logics, and by applying AFT,
Strass  \shortcite{Strass13} showed that many semantics
from Dung's argumentation frameworks and abstract dialectical frameworks \cite{DBLP:journals/ai/Dung95}
 can be obtained rather directly. More recently, AFT has been shown to play a key role in the study of semantics for database revision based on active integrity constraints \cite{BogaertsC18} and in addressing  semantics issues arising in 
weighted abstract dialectical frameworks, which are abstract argumentation frameworks that incorporate not only attacks but also support, joint attacks and joint support \cite{DBLP:conf/aaai/Bogaerts19}. AFT has also contributed to the study of induction \cite{DBLP:journals/ai/0001VD18} and knowledge compilation \cite{DBLP:journals/tplp/BogaertsB15}. 

In this paper, we add one more application to the above collection for hybrid MKNF (which stands for minimal knowledge and negation as failure).
Hybrid MKNF was proposed by Motik and Rosati \shortcite{Motik:JACM:2010} 
  for integrating nonmonotonic rules with description logics (DLs).
 Since reasoning with DLs is based on classic, monotonic logic, there is no support of nonmonotonic features such as defeasible inheritance or default reasoning.   On the other hand, rules under the stable model semantics \cite{GelfondLifschitz88} are formulated mainly to reason with ground knowledge, without supporting quantifiers or function symbols.
It has been argued that such a combination draws strengths from both and the weaknesses of one are balanced by the strengths of the other.  The formalism of hybrid MKNF knowledge bases provides a tight integration of rules with DLs.

A hybrid MKNF knowledge base $\K$ consists of two components, $\cK = (\cO,\cP)$, where $\cO$ is a DL knowledge base, which is expressed by a decidable first-order theory, and $\cP$ is a collection of MKNF rules based on the stable model semantics. 
MKNF structures in this case are two-valued, under which MKNF formulas are interpreted to be true or false. 
Knorr et al.  \shortcite{KnorrAH11} formulated a three-valued extension of MKNF and defined three-valued MKNF models, where the least one is called the {\em well-founded MKNF model}. An alternating fixpoint operator was then formulated for the computation of the well-founded MKNF model for (nondisjunctive) hybrid MKNF knowledge bases. In this paper, our primary goal is to show that this alternating fixpoint operator is in fact an approximator of AFT. Due to the abstraction power of AFT,  it turns out that Knorr et al.'s alternating fixpoint construction provides a uniform characterization of all semantics based on various kinds of three-valued MKNF models, including two-valued MKNF models of \cite{Motik:JACM:2010}.  

As shown in previous research \cite{KnorrAH11,LY17}, not all hybrid MKNF knowledge bases possess a well-founded MKNF model, and in general, deciding the existence of a well-founded MKNF model is intractable even if the underlying DL knowledge base is polynomial \cite{LY17}. On the other hand, we also know that  alternating fixpoint construction provides a tractable means in terms of a linear number of iterations 
to compute the well-founded MKNF model for a subset of hybrid MKNF knowledge bases. 
A question then is whether this subset can be enlarged. 
In this paper, we answer this question positively by formulating 
an improved approximator, which is {\em more precise} than the one derived from Knorr et al.'s alternating fixpoint operator. As a result, the well-founded MKNF model can be computed iteratively for a strictly  larger class of hybrid MKNF knowledge bases than what was known previously.

Hybrid MKNF combines two very different reasoning paradigms, namely closed world reasoning with nonmonotonic rules and open world reasoning with ontologies that are expressed in description logics. In this context, inconsistencies naturally arise. AFT was first developed for consistent approximations. 
In the seminal work \cite{DeneckerMT04}, the authors show that the theory of consistent approximations generalizes to a class of approximators beyond consistent pairs, which are called symmetric approximators. They also state that it is possible to develop a generalization of AFT without the symmetry assumption. These results and claims are given under the restriction that an approximator maps an exact pair 
on a product bilattice (which represents a two-valued interpretation)
to an exact pair. Unfortunately, this assumption is too restrictive for hybrid MKNF since a two-valued interpretation for a hybrid MKNF knowledge base may well lead to an inconsistent state.

Approximations under symmetric approximators already provide a powerful framework for characterizing intended models of a logic theory. But we want to go beyond that. We do not only want 
to capture consistent approximations in the product bilattice, but also want to allow operators to map a consistent state to an inconsistent one, and even allow inconsistent stable fixpoints. This is motivated by the possible role that AFT may play in building constraint propagators  for solvers of an underlying logic (e.g., \cite{Ji17}), where inconsistency not only guides the search via backtracking but also provides valuable information to prune the search space (e.g., by learned clauses in SAT/ASP solvers). One can also argue that inconsistent stable fixpoints may provide useful information for debugging purposes (a potential topic beyond the scope of this paper).

We show in this paper that all of the above requires only a mild generalization of AFT, which is defined for all pairs in the product bilattice without the assumption of symmetry. We 
relax the condition for an approximator so 
 that an approximator is required to map an exact pair to an exact pair only in the case of consistent approximation. 
Based on this revised definition of approximator, we present a definition of the stable revision operator, which is well-defined, increasing, and monotone on the product bilattice of a complete lattice, that guarantees existence of fixpoints and a least fixpoint.

In summary, we extend AFT from consistent and symmetric approximators to arbitrary approximators for the entire product bilattice. 
The goal is to use stable fixpoints as candidates for intended models, or to provide useful information on stable states (in terms of fixpoints) that may contain consistent as well as inconsistent information.  Such an extension is not without subtleties. We provide a detailed account of how such technical subtleties are addressed. 

The paper is organized as follows. The next section introduces notations, basics of fixpoint theory, and the current state of AFT. In Section \ref{AFT0}, we present an extended AFT. Section \ref{mknf} gives a review of three-valued MKNF and hybrid MKNF knowledge bases along with the underlying semantics. Then, in Section \ref{approximator-for-mknf} we show how Knorr et al.'s alternating fixpoint operator can be recast as an approximator and provide semantic characterizations, and in  Section \ref{richer0}, we show an improved approximator. Section \ref{related} is about related work, concluding remarks, and future directions. 

This paper is revised and extended from a preliminary report of the work that appeared in \cite{DBLP:conf/ruleml/LiuY19}.
The current paper is reorganized by first presenting a detailed study of generalized AFT. 
Especially, we provide an elaborate account of 
the original AFT and contrast it with our generalization. In this extended  version of the work, all claims are complete with a proof.

\section{Preliminaries}

In this section, 
we  recall the basic definitions regarding lattices underlying our work based on the Knaster-Tarski fixpoint theory \cite{Tarski1955}.  

A {\em partially ordered set} $\langle L, \leq\rangle $  is a set $L$ equipped  with a partial order $\leq$, which is a reflexive, antisymmetric,  and transitive relation. As usual, the strict order is expressed by $x < y$ as an abbreviation for $x \leq y$ and $x \not =y$.  
Given a subset $S \subseteq L$, an element $x \in L$ is an {\em upper bound} (resp. {\em a lower bound})  if $s \leq x$ (resp. $ x \leq s $) for all $s \in S$.
A {\em lattice} $\langle L, \leq\rangle $ is a {\em partially ordered set} (poset) in which every two elements have a {\em least upper bound} (lub) and a {\em greatest lower bound} (glb). A {\em complete lattice} is a lattice where every subset of $L$ has a least upper bound and a greatest lower bound. A complete lattice has both a least element $\bot$ and a greatest element $\top$.
A greatest lower bound of a subset $S \subseteq L$ is called a {\em meet} and a least upper bound of $S$ is called a {\em join}, and we use the notations: $\bigwedge S = \glb (S)$, $x \wedge y = \glb(\{x,y\})$, $\bigvee S = \lub(S)$, and $x \vee y = \lub(\{x,y\})$.
An operator $O$ on $L$ is {\em monotone} if for all $x,y \in L$,
that $x \leq y$ implies $O (x) \leq O(y)$.
An element $x \in L$ is a {\em pre-fixpoint} of ${O}$ if ${ O}(x) \leq x$; it is a {\em post-fixpoint} of  ${O}$ if $x \leq {O}(x)$. 
The Knaster-Tarski fixpoint theory \cite{Tarski1955} tells us the fact that a monotone operator $O$ on a complete lattice has fixpoints and a least fixpoint, denoted $\lfp(O)$, which coincides with its least pre-fixpoint.
The following result of  Knaster-Tarski fixpoint theory \cite{Tarski1955} serves as the basis of our work in this paper. 

\begin{theorem} Let $\langle L, \leq \rangle$ be a complete lattice and $O$ a monotone operator on $L$. Then $O$ has fixpoints, a least fixpoint, and a least pre-fixpoint. (i) The set of fixpoints of $O$ is a complete lattice under order $\leq$. (ii) The least fixpoint and least pre-fixpoint of $O$ coincide, i.e., $\lfp(O) = \wedge\{x \in L :O(x) \leq x \}$. \end{theorem}

A {\em chain} in a poset $\langle L, \leq \rangle$ is a linearly ordered subset of $L$. A poset $\langle L, \leq\rangle$ is {\em chain-complete} if it contains a least element $\bot$ and every chain $C \subseteq L$ has a least upper bound in $L$. A complete lattice is chain-complete, but the converse does not hold in general. However, as pointed out by \cite{DeneckerMT04},
the Knaster-Tarski fixpoint theory generalizes to chain-complete posets.

\begin{theorem}
\label{chain-complete0} {\rm \cite{chain-complete}} Let $\langle L, \leq \rangle$ be a chain-complete poset and $O$ a monotone operator on $L$. Then $O$ has fixpoints, a least fixpoint, and a least pre-fixpoint. (i) The set of fixpoints of $O$ is a chain-complete poset under order $\leq$. (ii)  The least fixpoint and  least pre-fixpoint of $O$ coincide.\end{theorem}

Given a complete lattice $\langle L, \leq\rangle $,  AFT is built on the induced product bilattice 
$\langle L^2, \leq_p\rangle$, where $\leq_p$ is called the {\em precision order} and defined as:
for all $x, y, x', y' \in L$,
$(x, y) \leq_p (x', y')$ if $x \leq x'$ and $y' \leq y$. The $\leq_p$ ordering is a complete lattice ordering on $L^2$. 
Below, we often write a lattice $\langle L, \leq\rangle $ by $L$ and its induced product bilattice by  $L^2$.

We define two {\em projection functions} for pairs in $L^2$: $(x,y)_1 = x$ and $(x,y)_2 = y$.  For simplicity, 
we write $A(x,y)_i$, where $i \in [1,2]$, instead of more formal $(A(x,y))_i$ to refer to the corresponding projection of the value of the operator $A$ on the pair $(x,y)$.
A pair $(x, y) \in L^2$ is {\em consistent} if $x \leq y$, {\em inconsistent} otherwise, and {\em exact} if $x = y$. 
A consistent pair $(x,y)$ in $L$ defines an {\em interval}, denoted $[x,y]$, which is identified by the set
$\{z~|~ x \leq z \leq y\}$. 
We therefore also use an interval to denote the corresponding set.  
A consistent pair $(x, y)$ in $L$ can be seen as an approximation of every $z \in L$ such that $z \in [x, y]$. In this sense, the precision order $\leq_p$ corresponds to the precision of approximation, while an exact pair approximates the only element in it. 
We denote  by $L^c$ the set of consistent pairs in $L^2$. Note that $\langle L^c, \leq_p \rangle$ is not a complete lattice in general.

On the other hand, an inconsistent pair $(x,y)$ in $L^2$ can be viewed as a  departure from some point $z \in L$, for which  $(z,z)$ is revised either by increasing the first component of the pair (w.r.t. the order $\leq$), or by decreasing its second component, or by performing both at the same time. Inconsistent pairs have a natural embedding of the notion of {\em the degree of inconsistency}. For two inconsistent pairs such that $(x_1, y_1)\leq_p (x_2,y_2)$, the latter is of higher degree of inconsistency than the former.  Here, there is a natural notion of inconsistency being {\em partial} as in contrast with full inconsistency represented by the special pair $(\top, \bot)$. Intuitively, this means that an inconsistent pair in general may embody consistent as well as inconsistent information. 

In logic programming for instance, $L$ is typically the power set $2^{\Sigma}$, where $\Sigma$ is a set of (ground) atoms representing reasoning individuals. A consistent pair $(T,P)$, where $T$ and $P$ are sets of atoms and $T \subseteq P$, is considered a three-valued interpretation, where $T$ is the set of true atoms and $P$ the set of possibly true atoms; thus the atoms in $\Sigma \setminus P$ are false.  If $T \not \subseteq P$,  the atoms that are in $T$ but not in $P$ are interpreted both true and false, resulting in inconsistency.  This gives rise to the notion of inconsistency in various degrees.

\subsection{Approximation fixpoint theory: the previous development}


At the center of AFT is the notion of approximator. We call an operator $A: L^2 \rightarrow L^2$ 
an {\em approximator} if $A$  is $\leq_p$-monotone and maps exact pairs to exact pairs. 
To emphasize the role of an operator $O: L \rightarrow L $ whose fixpoints are approximated by an approximator,
we 
say that  $A$ is 
an {\em approximator for $O$} if $A$ is $\leq_p$-monotone and $A(x,x) = (O(x), O(x))$ for all $x \in L$. 

In \cite{DeneckerMT04}, 
AFT was first developed for consistent approximations, where an approximator is {\em consistent} if it maps consistent pairs to consistent pairs. 
We denote by  $\Approx(L^2)$ the set of all approximators on $L^2$ and by $\Approx(L^c)$ the set of consistent approximators on $L^c$.  Given an approximator $A \in \Approx(L^2)$,  we denote by $A^c$ the restriction of $A$ to $L^c$ under the condition that $A^c$ is an operator on $L^c$.\footnote{Such $A^c$ may not exist in general, but for symmetric approximators, it always does; cf. Proposition 14 of \cite{denecker2000approximations}.} 

For the study of semantics based on partial interpretations, we can focus on the fixpoints of approximators, independent of how they may approximate operators  on $L$.  First, since $\langle L^c, \leq_p \rangle$ is not a complete lattice, the Knaster-Tarski fixpoint theory does not apply. But $L^c$ is a chain-complete poset (ordered by $\leq_p$), so according to Markowsky's theorem, an approximator $A \in \Approx(L^c)$ has a least fixpoint, called {\em Kripke-Kleene fixpoint} of $A$, and other fixpoints. However, some of these fixpoints may not satisfy the minimality principle commonly adopted in knowledge representation.\footnote{The situation is analogue to the notion of Kripke-Kleene model of a logic program, which is a least fixpoint of a 3-valued van Emden-Kowalski operator.}  To eliminate non-minimal fixpoints, we can focus on what are called the  {\em stable fixpoints} of $A$, which 
are the fixpoints of a {\em stable revision operator} $St_A: L^c \rightarrow L^c$, which is defined as:
 \begin{eqnarray}
 St_A(u,v) = 
(\lfp({A}( \cdot , v)_1), \lfp(A(u, \cdot)_2))
\label{stable-revision0}
\end{eqnarray}
where ${A}( \cdot , v)_1$ denotes the operator $[\bot, v] \rightarrow [\bot, v]: z \mapsto A(z, v)_1$ and $A(u,\cdot)_2$ denotes the operator $[u,\top] \rightarrow [u,\top]: z \mapsto A(u,z)_2$.

Denecker et al. \shortcite{DeneckerMT04} show that (\ref{stable-revision0}) is well-defined for pairs in $L^c$ under a desirable property. We call a pair $(u,v) \in L^c$ {\em $A$-reliable} if $(u,v) \leq_p A(u,v)$. Intuitively, if $A(u,v)$ is viewed as a revision of $(u,v)$ for more accurate approximation, under $A$-reliability, $A(u,v)$ is at least as accurate as $(u,v)$.
 Furthermore,
Denecker et al. \shortcite{DeneckerMT04} show that 
if a pair $(u,v) \in L^c$ is $A$-reliable, then $A(\cdot , v)_1$ is internal in $[\bot, v]$, hence we can safely define $A(\cdot,v)_1$ to be an operator on the lattice $[\bot, v]$; similarly, since $A(u,\cdot)_2$ is internal in $[u,\top]$, we can define it on lattice $[u,\top]$ (Proposition 3.3). Since the operators $A(\cdot,v)_1$ and $A(u,\cdot)_2$ are $\leq$-monotone on their respective domains, a least fixpoint for each exists; hence the stable revision operator $St_A$ is well-defined. 
Note that by definition, since a fixpoint of $St_{A}$ is a fixpoint of $A$, a stable fixpoint of $A$ is a fixpoint of $A$.

However, the notion of $A$-reliability is not strong enough to guarantee another desirable property: for any $A$-reliable pair $(u,v)$, we want $(u,v) \leq_p St_{A}(u,v) (= (\lfp({A}( \cdot , v)_1), \lfp(A(u, \cdot)_2)))$; i.e., a stable fixpoint computed from a given pair should be at least as accurate.
This property does not hold in general for $A$-reliable pairs. In addition, we also want 
$A(u,v) \leq_p St_{A}(u,v$), so that there is a guarantee that the stable revision operator ``revises even more", i.e., stable revision is at least as accurate as revision by a single application of $A$.
We therefore introduce a new property: 
an $A$-reliable pair $(u,v) \in L^c$ is called {\em $A$-prudent} if $u \leq \lfp({A}( \cdot , v)_1)$. We denote by $L^{rp}$
the set of 
$A$-prudent pairs in $L^c$. Denecker et al. \shortcite{DeneckerMT04} show that for all $A$-prudent pairs $(u,v)$ in $L^c$, $(u,v)
\leq_p St_{A}(u,v)$ and $A(u,v)
\leq_p St_{A}(u,v)$
(Propositions 3.7 and 3.8). 

\begin{example}
Consider a complete lattice $\langle L, \leq\rangle$ where 
$L = \{\bot, \top\}$ and $\leq$ is defined as usual. Define an operator $A$ on $L^c$ as: 
$A(\top, \top) = (\top, \top)$ and $A(\bot, \top) = A(\bot, \bot) = (\top, \top)$. It can be seen that 
$A$ is $\leq_p$-monotone on $L^c$, the pairs $(\top, \top)$ and $(\bot, \top)$ are $A$-reliable, and $(\bot, \bot)$ is not.
Both $A$-reliable pairs $(\top, \top)$ and $(\bot, \top)$ are $A$-prudent as well, thus $L^{rp} = \{(\top, \top), (\bot, \top)\}$. 

Now let $A'$ be the identify operator on $L^c$ except $A'(\bot,\bot) = (\top, \top)$. The operator $A'$ is $\leq_p$-monotone on $L^c$. The pairs $(\top, \top)$ and $(\bot, \top)$ are $A'$-reliable whereas $(\bot, \bot)$ is not.
 But the $A'$-reliable pair
$(\top,\top)$ is not $A'$-prudent because $\lfp({A'}( \cdot , \top)_1) = \bot < \top$.  Note that $(\top,\top)$ is a fixpoint of $A'$ but not a stable fixpoint. Thus, for approximator $A'$, $L^{rp} = \{(\bot, \top)\}$.
\end{example}

The above development has led to the following results of the properties of the stable revision operator.

\begin{theorem}  [Theorem 3.11 of \cite{DeneckerMT04}]
Let $L$ be a complete lattice, $A \in \Approx(L^c)$. The set of $A$-prudent elements of $L^c$ is a chain-complete poset under the precision order $\leq_p$, with least element $(\bot,\top)$. The stable revision operator is a well-defined, increasing and monotone operator in this poset.
\label{chain-complete}
\end{theorem}

This theorem  serves as the foundation for AFT as it guarantees that the stable revision operator has fixpoints and a least fixpoint, which we have called stable fixpoints of $A$. 

The notion of approximator is  then generalized to {\em symmetric approximators}, which are $\leq_p$-monotone operators $A$ on $L^2$ such that $A(x,y)_1 = A(y,x)_2$, for all $x, y \in L$. As remarked in 
\cite{DeneckerMT04}, this generalization is motivated by operators arising in knowledge 
representation that are symmetric.\footnote{For example,  Fitting's immediate consequence operator 
 for normal logic programs 
 \cite{Fitting02}, placed in the context of bilattice $((2^{\Sigma})^2, \subseteq_p)$ where $\Sigma$ is a set of ground atoms, induces 
a symmetric approximator.}
A critical property of a symmetric approximator $A$  is that
$A(x,x)$ yields an exact pair, for all $x \in L$, i.e., it maps an exact pair to an exact pair, which is consistent. This can be seen  as follows: 
Since $A(x,x) = (A(x,x)_1, A(x,x)_2)$ for all $x \in L$, and by the symmetry of $A$, $A(x,x)_1 = A(x,x)_2$ and thus $A(x,x)$ is consistent.



\section{Approximation Fixpoint Theory Generalized}
\label{AFT0}
In this section, we generalize AFT as given in \cite{DeneckerMT04} 
from consistent and symmetric approximators to arbitrary approximators.
 This generalization is needed in order to define approximators for hybrid MKNF knowledge bases since an exact pair in this context is a two-valued interpretation which can be mapped to an inconsistent one. This is because 
 a hybrid MKNF knowledge base allows predicates to appear both in the underlying DL knowledge base and in rules, inconsistencies may arise from the combination of classic negation in the former and derivations using nonmonotonic negation in the latter.

The current AFT is defined for consistent and symmetric approximators. As alluded earlier, 
a critical property of a symmetric approximator  is that it maps an exact pair to an exact pair. 
However, a $\leq_p$-monotone operator on $L^2$ may not possess this property.

\begin{example}  \cite{BiYF14}
Consider a complete lattice where ${L}= \{\bot, \top\}$ and $\leq$ is defined as usual. Let ${O}$ be the identity function on
${L}$. Then we have two fixpoints, ${O}(\bot) = \bot$ and ${O}(\top) = \top$.
Let ${A}$ be an identity function on ${L}^2$ everywhere except ${A}(\top,\top) = (\top, \bot)$.  Thus, ${A}(\top,\top)$ is inconsistent.
It is easy to check that ${A}$ is $\leq_p$-monotone, especially, from $(\top, \top) \leq_p (\top, \bot)$ we have ${A}(\top, \top) \leq_p { A}(\top, \bot)$. There is exactly one exact pair 
$(\bot,\bot)$ for which ${A}(\bot, \bot)$ is consistent, and the condition ${A}(\bot, \bot) = ({O}(\bot), {O}(\bot))$ is satisfied. For the other exact pair $(\top,\top)$, $A(\top,\top)$ is inconsistent 
and ${A}(\top, \top) \not = ({O}(\top), O(\top))$, even though ${O}(\top) = \top$. The fixpoint $\top$ of ${ O}$ is not captured by the operator ${A}$ because ${A}(\top, \top)$ is inconsistent.

\medskip
\noindent
{\bf  Conclusion:}  Though the operator $A$ above is $\leq_p$-monotone on $L^2$, it is not an approximator by the current definition because it fails to map an exact pair to an exact pair when inconsistency arises.\footnote{This example specifies a system  in which states are represented by a pair of factors \-- high and low. Here, all states are stable except the one in which both factors are high.  This state may be transmitted to an ``inconsistent state" with the first factor  high and the second low.
This state is the only inconsistent one, and it itself is stable.
}
\label{example1}
\end{example}

In order to accommodate operators like $A$ above, we present a generalization by relaxing the condition for an approximator.
\begin{definition}
We say that an operator $A: {L}^2\rightarrow {L}^2$ is an approximator 
if $A$ is $\leq_p$-monotone and for all $x \in {L}$, if ${A}(x,x)$ is consistent then $A$ maps $(x,x)$ to an exact pair. Let $O$ be an operator on $L$. 
We say that 
$A: L^2 \rightarrow L^2$ is an {\em approximator for $O$}
if $A$ is an approximator and for all $x \in {L}$, if ${A}(x,x)$ is consistent then 
 ${A}(x,x) = ({O}(x), {O}(x))$.
\end{definition}

 That is, we make the notion of approximation partial:  $A(x,x)$ captures $O$ only when $A(x,x)$ is consistent. Under this definition, the operator $A$ in Example \ref{example1} is an approximator and it approximates, for example, the identify operator $O$ on $L$.

Before we proceed to generalize the notion of stable revision operator,  we need to extend the definition of $A$-reliability and $A$-prudence to pairs in $L^2$. Such a definition is already provided in the study of well-founded inductive definitions.
Following  \cite{Denecker07}, given an operator $A$ on $L^2$,
we say that a pair $(u, v) \in {L}^2$ is {\em ${A}$-contracting} if $(u, v) \leq_p {A}(u, v)$.\footnote{Earlier in this paper, 
$A$-contracting pairs were called $A$-reliable in the context of $L^c$.} 
The notion of $A$-prudence is generalized to $L^2$ as well. 
 A pair $(u, v) \in {L}^2$ is {\em ${A}$-prudent}  if $u\leq \lfp({A}(\cdot, v)_1)$ (when $ \lfp({A}(\cdot, v)_1)$ exists). By an abuse of notation and without confusion, in the rest of this paper  we will continue to use ${L}^{rp}$ but this time to denote the set of ${A}$-contracting and ${A}$-prudent pairs in ${L}^2$.

Now, we relax the definition of the stable revision operator as follows: Given any pair $(u,v) \in L^2$, define
 \begin{eqnarray} St_A(u,v) = 
(\lfp({A}( \cdot , v)_1), \lfp(A(u, \cdot)_2))
\label{stable-revision}
\end{eqnarray}
where ${A}( \cdot , v)_1$ denotes the operator $L \rightarrow L: z \mapsto A(z, v)_1$ and $A(u,\cdot)_2$ denotes the operator $L \rightarrow L: z \mapsto A(u, z)_2$. That is, both 
${A}( \cdot , v)_1$ and $A(u,\cdot)_2$ are operators on $L$.

\medskip
\noindent
{\bf Notation:}  Let $(u,v) \in L^2$ and $A  \in  \Approx(L^2)$. We define
$$(C_1(v), C_2(u)) =  (\lfp({A}( \cdot , v)_1), \lfp(A(u, \cdot)_2))$$ where  ${A}( \cdot , v)_1$ and $A( u, \cdot )_2$ are the respective  projection operators defined on $L$. We use the notation 
$(C_1(v), C_2(u))$ with the understanding that the underlying approximator is clear from the context.

Since $A$ is $\leq_p$-monotone on $L^2$, the projection operators ${A}( \cdot , v)_1$ and $A( u, \cdot )_2$, for any pair $(u,v) \in L^2$, are both $\leq$-monotone on $L$, which guarantees the existence of a least fixpoint for each. Thus, the stable revision operator in equation (\ref{stable-revision}) is well-defined for all pairs in $L^2$.  Note that in this case a stable fixpoint can be inconsistent. 
For example, consider lattice $L = \{\bot,\top\}$ and an operator $A$ on $L^2$, which is identity on every pair except $A(\bot,\bot) = (\top,\bot)$. Clearly, $A$ is $\leq_p$-monotone.  The inconsistent pair $(\top,\bot)$ is a stable fixpoint of $A$ since 
$St_A(\top,\bot) = 
(\lfp({A}( \cdot , \bot)_1), \lfp(A(\top, \cdot)_2)) = (\top, \bot)$. 

The definition of stable revision above has been proposed and adopted in the literature of AFT already,\footnote{But notice a critical difference in our definition of an approximator discussed above.}  e.g., in \cite{DeneckerMT04} and more recently in \cite{BogaertsC18,Bogaerts-denecker-AIJ-15}, for consistent and symmetric approximators. It however differs from stable revision for consistent approximators with regard to 
 the domains of the two projection operators. As mentioned earlier, 
in consistent AFT (where an approximator is from $\Approx(L^c)$), we know from \cite{DeneckerMT04}
that
${A}( \cdot , v)_1$ is internal in $[\bot, v]$ 
so we define 
${A}( \cdot , v)_1$ to be an operator on the lattice $[\bot,v]$, and  $A( u, \cdot )_2$ is internal in $[u,\top]$ so we define it 
 on lattice $[u,\top]$. Now, let us generalize this to all approximators in $\Approx(L^2)$ for consistent pairs in $L^c$.

\medskip
\noindent
{\bf Notation:}  Let $(u,v) \in L^c$ and $A \in \Approx(L^2)$. We define $$(D_1(v), D_2(u))  =  (\lfp({A}( \cdot , v)_1), \lfp(A(u, \cdot)_2))$$ where  ${A}( \cdot , v)_1$ is defined on $[\bot, v]$ and $A( u, \cdot )_2$ is defined on $[u,\top]$. In the sequel, the term {\em consistent stable fixpoints} refer to the fixpoints determined by this definition.

Since we consider the entire product bilattice $L^2$, we are interested in knowing which consistent pairs in it make the above projection operators well-defined under our relaxed definition of approximators.

\begin{proposition}
Let $\langle {L}, \leq\rangle$ be a complete lattice and
${A}$ an approximator on $L^2$. 
If a consistent pair $(u,v) \in {L}^2$ is $A$-contracting and $A(u,u)$ is consistent, 
then
  for every 
$x \in [u, \top]$,  ${A}(u,x)_2 \in [u, \top]$, and 
for every $x \in [\bot,v]$, $A(u,v)_1 \in [\bot, v]$.
\label{3op}
\end{proposition}

\begin{proof}
We can show that, for any $ x \in[u,\top]$,
$$u \leq A(u,v)_1 \leq A(u,u)_1 = A(u,u)_2 \leq A(u,x)_2.$$
The first inequality is because $(u,v) \leq_p {A}(u,v)$ (i.e., $(u,v)$ is  $A$-contracting). The second is due to
${A}(u,v) \leq_p {A}(u,u)$, as $(u,v)$ is consistent thus
$(u,v) \leq_p (u,u)$ and ${A}$ is $\leq_p$-monotone.
The next equality is by the fact that since ${A}(u,u)$ is consistent, it maps a consistent pair to a consistent pair. The last inequality is due to $x \geq u$ and that ${A}$ is $\leq_p$-monotone. For any $ x \in[\bot, v]$, we can similarly show that $A(x,v)_1 \leq A(v,v)_1 = A(v,v)_2 \leq v$.\footnote{The proof is essentially the same as the poof of Proposition 3.3 in  \cite{DeneckerMT04}; but there is a subtle difference in the definition of approximator: in the case of  \cite{DeneckerMT04}, the claim is proved for $A^c \in  \Approx(L^c)$. But in our case,
the claim is for arbitrary approximators in $ \Approx(L^2)$.
This shows an argument in favor of our relaxed definition of approximators.
}
\end{proof}

\comment{
\begin{proof}
To prove it, we can show the following
$$u \leq A(u,v)_1 \leq A(u,u)_1 = A(u,u)_2\leq A(u,x)_2.$$
The first inequation is because $(u,v) \leq_p  {A}(u,v)$. The second is due to $\leq_p$-monotonicity of $A$, i.e., 
${A}(u,v) \leq_p {A}(u,u)$. The next equation is by the fact that $A$ is an approximator and $A(u,u)$ is consistent.  The last inequation because due to $x \geq u$ and that ${A}$ is $\leq_p$-monotone.
\end{proof}
}

A question that arises is whether consistent stable fixpoints from consistent approximations are carried over to approximators on $L^2$. That is, assume $(u,v) \in L^c$ is a stable fixpoint as computed by 
$(D_1(v), D_2(u))$, and the question is whether $(u,v)$ is also a stable fixpoint as computed by $(C_1(v), C_2(u))$. 
If $(D_1(v), D_2(u)) = (C_1(v), C_2(u))$, then the answer is yes for $(u,v)$.  In this way, a consistent stable fixpoint as computed  by $(D_1(v), D_2(u))$
is preserved for the stable revision operator as defined by $(C_1(v), C_2(u))$.

The above question was answered positively by \cite{DeneckerMT04} (cf. Theorem 4.2) for symmetric approximators by restricting them to consistent pairs. The authors show that the theory of consistent approximations captures general AFT that treats consistent and symmetric approximators on the product bilattice, 
as long as we restrict our attention to consistent pairs.  
They show 
that for any symmetric approximator $A$, a consistent pair $(u,v)$ is a stable fixpoint of $A$ on $L^2$ (as defined in terms of  $(C_1(v), C_2(u))$)
if and only if it is a stable fixpoint of $A^c$  (as defined in terms of $(D_1(v), D_2(u))$).
They state that it is possible to develop a generalization of AFT for which these results hold without the assumption of symmetry. 
However, once we allow consistent pairs to be mapped to inconsistent ones and adopt the domain $L$ for the projection operators, a discrepancy with consistent AFT emerges. 


\begin{example}
\label{why}
Let $L = 
\{\bot, \top\}$ and $A$ an identity function everywhere on $L^2$ except that  $A(\bot,\top) = A(\bot, \bot) = (\top,\top)$. It is easy to verify that 
$A$ is $\leq_p$-monotone. 
Clearly, $A^c \in \Approx(L^c)$, i.e., it maps consistent pairs to consistent pairs, 
 it is $\leq_p$-monotone on $L^c$,  and approximates, e.g., the identify operator $O$ on $L$. But $A$ is not symmetric since $A(\bot,\top)_1 = \top$ and $A(\top,\bot)_2 = \bot$.
Since $A^c \in \Approx(L^c)$, 
$A^c(\top, \cdot)_2$ is an operator on $[\top,\top]$. Since $St_{A^c}(\top,\top) = (\lfp(A^c(\cdot, \top)_1), \lfp(A^c(\top, \cdot)_2) 
 = (\top, \top)$, it follows that $(\top,\top)$ is a stable fixpoint of $A^c$. Now let us apply the definition of stable revision in equation (\ref{stable-revision}) to approximator $A$, where both projection operators $A(\cdot, y)_1$ and $A(x,\cdot)_2$ 
are defined on $L$. In this case, since $St_{A}(\top,\top) = (\lfp(A(\cdot, \top)_1), \lfp(A(\top, \cdot)_2) 
 = (\top, \bot)$, $(\top,\top)$ is not a stable fixpoint of $A$.
This example is not a surprise since in general different domains may well lead to different least fixpoints. 

Now consider another approximator $A'\in \Approx(L^2)$ such that $A'$ maps all pairs to $(\top,\top)$.  It can be seen that $A'$ is $\leq_p$-monotone and $(\top,\top)$ is a stable fixpoint of $A'$ in both cases, where $A'(\cdot, T)_1$ is defined as an operator either on $[\bot, \top]$ or on $L$, and 
$A'(\top, \cdot)_2$ is defined as an operator either on $[\top,\top]$ or on $L$. 
That is, for each projection operator, the least fixpoints of it on two different domains coincide. 

\smallskip
\noindent
{\bf Conclusion:} For an arbitrary approximator $A$ on the product bilattice $L^2$, the stable revision operator $St_A(u,v)$ is well-defined for all pairs $(u,v) \in L^2$, if we define both projection operators on $L$. However, consistent stable fixpoints under consistent AFT may not be preserved  if we adopt the stable revision operator as defined in this paper (i.e., by equation (\ref{stable-revision})  in terms of $(C_1(v), C_2(u))$).
\end{example}

Let us call the existence of a gap between the two pairs of least fixpoints, $(C_1(v), C_2(u))$ and $(D_1(v), D_2(u))$, discussed above an ``anomaly".  One can argue that a desirable approximator should not exhibit this anomaly so that accommodating inconsistent pairs does not have to sacrifice the preservation of  consistent stable fixpoints.

\begin{definition}
\label{strong}
 Let $A \in \Approx(L^2)$, and $(u,v) \in L^c$ such that 
$(u,v) = (\lfp(A(\cdot, v)_1), \lfp(A(u,\cdot)_2))$ 
where  $A(\cdot,v)_1$ is an operator on $[\bot,v]$ and 
$A(u,\cdot)_2$ is an operator on $[u,\top]$. Approximator $A$ is called {\em strong} for $(u,v)$ if 
$(u,v) = (\lfp(A(\cdot, v)_1), \lfp(A(u,\cdot)_2))$ 
where  both $A(\cdot,v)_1$ and $A(u,\cdot)_2$ are operators on $L$. 
Approximator $A$ is called {\em strong} if it is strong for every $(u,v) \in L^c$ that satisfies the above condition.
\end{definition}

In other words, a strong approximator preserves consistent stable fixpoints under the definition of stable revision adopted in this paper. 
For example, in Example \ref{why}, while the approximator $A'$ is strong for $(\top, \top)$, the approximator $A$ is not.

A question arises: are there natural approximators that are strong? For normal logic programs, it is known that 
Fitting's immediate consequence operator  $\Theta_{\cal P}$ \cite{Fitting02} induces a symmetric approximator. It can be 
shown that $\Theta_{\cal P}$ is also a strong approximator.\footnote{We can apply Lemma 4.1 of \cite{DeneckerMT04}, which says that for any symmetric approximator $A$ and for any consistent pair $(u,v)$, if $(u,v)$ is $A^c$-prudent, then $(D_1(v), D_2(u)) = (C_1(v), C_2(u))$. Since a consistent,  $\Theta_{\cal P}$-prudent
stable fixpoint of $\Theta_{\cal P}$ is $\Theta_{\cal P}^c$-prudent, the conclusion follows.}
In addition, we show later in this paper that the approximators we formulate  for hybrid MKNF knowledge bases are (essentially) strong approximators.\footnote{Technically, we need a mild condition: Given a consistent stable fixpoint $(u,v)$, these approximators are strong for $(u,v)$
if $u$ is consistent with the given DL knowledge base. If the condition is not satisfied, the stable fixpoint $(u,v)$ does not correspond to a three-valued MKNF model. Thus, the condition does not affect the preservation of consistent stable fixpoints that give three-valued MKNF models.} If we focus on strong approximators, the relaxed AFT as presented in this paper can be seen as a generalization of the original AFT.



Finally, as a generalization of the current AFT, we show that the properties of the stable revision operator as stated in Theorem \ref{chain-complete} for consistent AFT can be generalized. 




 \begin{theorem} 
\label{chain}
 Let $({L}, \leq)$ be a complete lattice and ${A}$ an approximator on $L^2$. Then, $\langle {L}^{rp}, \leq_p \rangle$ is a chain-complete poset under the precision order $\leq_p$, with least element $(\bot,\top)$. The stable revision operator as given in equation (\ref{stable-revision}) is a well-defined, increasing and monotone operator in this poset.  
\end{theorem}

\begin{proof} 
The least element $(\bot,\top)$ is naturally ${A}$-contracting and ${A}$-prudent. 
Let $C$ be a chain in ${L}^{rp}$, and $C_1$, $C_2$ be the respective projections of $C$.  
First we show that the element $(\lub(C_1), \glb(C_2)) = (\bigvee C_1, \bigwedge C_2)$ 
is the least upper bound of $C$, which is also in $L^{rp}$. 
Since $C$ is a chain in $L^{rp}$ ordered by the relation $\leq_p$, it is easy to see that the least upper bound of $C_1$ exists, which is just the maximum element in $C_1$;  similarly, the greatest lower bound of $C_2$ exists.  
Then, it is clear that the least upper bound of $C$ is $\lub(C) = (\lub(C_1), \glb(C_2))$.

To show $(\lub(C_1), \glb(C_2)) = (\bigvee C_1, \bigwedge C_2)$ is ${A}$-contracting and ${A}$-prudent, 
let $u_0=\bigvee C_1$ and $v_0=\bigwedge C_2$ and consider any $(a,b) \in C$. Since $C$ is a chain in $L^{rp}$ that contains $(a,b)$, $a \in C_1$ and $b \in C_2$, we have $a \leq u_0$ and $v_0 \leq b$, from which we obtain $a \leq {A}(a,b)_1 \leq {A}(u_0,b)_1 \leq {A}(u_0,v_0)_1$, where the first inequality is because $(a,b) \in L^{rp}$ is $A$-contracting and the next two inequalities are due to the $\leq_p$-monotonicity of $A$. 
Since $a \in C_1$ is arbitrary,  letting $ a = u_0$, we then have $u_0=\bigvee C_1 \leq {A}(u_0,v_0)_1$.
Similarly, we can show ${A}(u_0,v_0)_2 \leq {A}(a,v_0)_2 \leq {A}(a,b)_2 \leq b$. Since $v_0=\bigwedge C_2$, it  follows that ${A}(u_0,v_0)_2 \leq v_0=\bigwedge C_2$.  Hence, $(\bigvee C_1, \bigwedge C_2)$ is ${A}$-contracting. 

To show that $(\bigvee C_1, \bigwedge C_2)$ is $A$-prudent, 
let $u'=\lfp({A}(\cdot, v_0)_1)$. For any $(a,b) \in C$, we have ${A}(u', b)_1 \leq {A}(u', v_0)_1=u'$, then $u'$ is a pre-fixpoint of ${A}(\cdot,b)_1$ and thus $\lfp({A}(\cdot, b)_1) \leq u'$.
Also since $(a,b)$ is $A$-contracting, we have $a \leq \lfp({A}(\cdot, b)_1) \leq u'$. Since $a$ is arbitrary from $C_1$, this applies to 
$u_0 = \bigvee C_1 \in C_1$ 
and thus $\bigvee C_1 \leq u'$. That is, $(\bigvee C_1, \bigwedge C_2)$ is ${A}$-prudent. 

We therefore conclude that 
$(L^{rp}, \leq_p)$ is a chain-complete poset under order $\leq_p$. 

\comment{
Next we can show by transfinite induction that 
for every element $(a,b) \in L^{rp}$, $A(a,b) \in L^{rp}$. Given $(a,b)\in L^{rp}$, let $A(a,b)=(c,d)$. Since $a \leq c$ and $d \leq b$, we have $c=A(a,b)_1 \leq A(c,b)_1 \leq A(c,d)_1$ and $A(c,d)_2 \leq A(c,b)_2 \leq A(a,b)_2=d$, and therefore $(c,d)$ is $A$-contracting.

For showing $(c,d)$ is $A$-prudent, let $u_1=\lfp(A(\cdot,b)_1)$ and $u_2=\lfp(A(\cdot,d)_1)$. We then derive $c=A(a,b)_1 \leq A(u_1,b)_1=u_1$, i.e., $c \leq u_1$, since $(a,b)$ is $A$-prudent, namely $a \leq u_1=\lfp(A(\cdot,b)_1)$. Since both $A(\cdot,b)_1$ 
and $A(\cdot, d)_1$ are defined on $L$, we have $u_1=\lfp(A(\cdot,b)_1) \leq u_2=\lfp(A(\cdot,d)_1)$ by $d \leq b$, and therefore $c \leq u_2=\lfp(A(\cdot,d)_1)$, i.e., $(c,d)$ is $A$-prudent.
}

Next, we show that the stable revision operator defined in equation (\ref{stable-revision}) is a well-defined, increasing and monotone operator in this poset.
In the definition of the stable revision operator $St_A(u,v) = 
(\lfp({A}( \cdot , v)_1), \lfp(A(u, \cdot)_2))$ in equation (\ref{stable-revision}), we already argued that $St_{A}$ is a well-defined mapping, due to the fact that both projection operators 
${A}( \cdot , v)_1$ and $A(u, \cdot)_2$
are defined on $L$.  We now show that 
\begin{itemize}
\item 
$St_{A}$ is well-defined for $L^{rp}$, namely $L^{rp}$ is 
closed under $St_{A}$, i.e., for all $(a,b) \in L^{rp}$, $St_A(a,b) \in L^{rp}$, 
\item 
$St_{A}$ is increasing, i.e., $(a,b) \leq_p  St_{A}(a,b)$ for all $(a,b) \in L^{rp}$, 
 and 
\item
$St_{A}$ is $\leq_p$-monotone. 
\end{itemize}

Let $(a,b), (c,d) \in L^{rp}$. For simplicity, let $u_1=St_{A}(a,b)_1=\lfp(A(\cdot,b)_1)$ and $v_1=St_{A}(a,b)_2=\lfp(A(a,\cdot)_2)$, $u_2=St_{A}(c,d)_1=\lfp(A(\cdot,d)_1)$ and $v_2=St_{A}(c,d)_2=\lfp(A(c,\cdot)_2)$.

For convenience, let us first show that $St_{A}$ is increasing and $\leq_p$-monotone.
By $A$-prudence of $(a,b)$, $a \leq u_1$. Since $(a,b)$ is $A$-contracting, $A(a,b)_2 \leq b$. Thus 
 $b$ is a pre-fixpoint of $A(a, \cdot)_2$, and since $v_1=\lfp(A(a,\cdot)_2)$, it follows $v_1 \leq b$. That is,  $(a,b) \leq_p  St_{A}(a,b)$. 

For $\leq_p$-monotonicity, given $(a,b) \leq_p (c,d)$, we have $A(u_2,b)_1 \leq A(u_2,d)_1=u_2$ by $d \leq b$, and thus $u_2$ is a pre-fixpoint of $A(\cdot,b)_1$ and $u_1 \leq u_2$. Similarly, $A(c,v_1)_2 \leq A(a,v_1)_2=v_1$ by $a \leq c$, so
$v_1$ is a pre-fixpoint of $A(c,\cdot)_2$ and thus $v_2 \leq v_1$. That is, $St_{A}(a,b) \leq_p St_{A}(c,d)$. 

We now show that $St_A$ maps a pair $(a,b) \in L^{rp}$ to a pair in $L^{rp}$. We observe that $u_1 = A(u_1,b)_1 \leq A(u_1, v_1)_1$, where the equality is because $u_1$ is a fixpoint of the operator $A(\cdot, b)_1$ and the inequality is because $(a,b)$ is $A$-contracting. Similarly, $v_1 = A(a,v_1)_2 \leq A(u_1, v_1)_2$. Therefore, $(u_1,v_1) \leq_p A(u_1, v_1)$, i.e., $(u_1,v_1)$ is $A$-contracting. 
To prove $A$-prudence of $(u_1,v_1)$, since $(a,b)$ is $A$-contracting, $b \geq v_1$ and by $\leq_p$-monononitcity 
of $A$, for any $x \in L$, $A(x,b) \leq_p A(x,v_1)$ and thus $A(x,b)_1 \leq A(x,v_1)_1$. Thus, every pre-fixpoint of $A(\cdot,v_1)_1$ is a pre-fixpoint of $A(\cdot, b)_1$, i.e., for any $z \in L$, if $A(z,v_1)_1 \leq z$ then $A(z,b)_1 \leq A(z, v_1)_1 \leq z$. 
Since $A(\cdot, v_1)_1$ is a monotone operator on $L$, $\lfp(A(\cdot, v_1)_1)$ exists and thus the set of pre-fixpoints of $A(\cdot, v_1)_1$ is nonempty. Therefore $u_1 \leq \lfp(A(\cdot, v_1)_1)$ and $(u_1,v_1)$ is $A$-prudent. 
\end{proof}

We now can apply Theorem \ref{chain-complete0} so that given an approximator $A$ on the product bilattice $L^2$, the stable revision operator defined by equation (\ref{stable-revision}) possesses fixpoints and a least fixpoint, the latter of which can be computed iteratively from the least element $(\bot, \top)$.

Note that by the Knaster-Tarski fixpoint theory, since $\langle L^2, \leq_p \rangle$ is a complete lattice and the stable revision operator 
$St_{A}$ is $\leq_p$-monotone on $L^2$ (which can be shown by the same proof for the $\leq_p$-monotonicity on $L^{rp}$ above),
the operator $St_{A}$ defined in equation (\ref{stable-revision}) is already guaranteed to possess fixpoints and a least fixpoint.  Nevertheless, Theorem \ref{chain} above is still relevant because it shows a generalization of the chain-completeness result from $L^c$ to $L^2$, and in addition, it points to a smaller domain of pairs $L^{rp}$ from which consistent as well as inconsistent stable fixpoints can be computed by the guess-and-verify method.


\section{Hybrid MKNF Knowledge Bases}
\label{mknf}
\subsection{Minimal knowledge and negation as failure}
The logic of minimal knowledge and negation as failure (MKNF) \cite{DBLP:conf/ijcai/Lifschitz91} is based on a first-order language~$\cL$ (possibly with equality $\approx$) with two modal operators, $\boldK$, for minimal knowledge, and $\bfnot$, for negation as failure.
In MKNF,  {\em first-order atoms} are defined as usual and 
 {\em MKNF  formulas} are first-order formulas with $\bfK$ and $\bfnot$. 
An MKNF formula $\varphi$ is {\em ground} if it contains no variables,  and $\varphi[t/x]$ denotes the formula obtained from $\varphi$ by replacing all free occurrences of variable $x$ with term $t$.
 Given a first-order formula $\psi$, $\boldK \psi$ is called a (modal) {\em $\bfK$-atom} and 
$\boldnot \psi$ called a (modal) {\em $\bfnot$-atom}. Both of these are also called {\em modal atoms}.

A {\em first-order interpretation} is understood as in first-order logic. The universe of a first-order interpretation~$I$ is denoted by $\left|I\right|$.
A {\em first-order structure} is a nonempty set $M$ of first-order interpretations with the universe $\left|I\right|$ for some fixed $I\in M$.
An {\em MKNF structure} is a triple $(I, M, N)$, where $M$ and $N$ are sets of first-order interpretations with the universe $\left|I\right|$. We extend the language $\cL$ by adding object constants representing all elements of $\left|I\right|$, and call these constants {\em names}. 
The satisfaction relation $\models$ between an MKNF structure $(I, M, N)$ and an MKNF formula $\varphi$ is defined as follows: 
$$
\begin{array}{ll}
(I, M, N)\models \varphi ~(\varphi \mbox{ is a first-order atom) if } \varphi \mbox{ is true in } I,  \\
(I, M, N)\models \neg \varphi \mbox{ if } (I, M, N)\not\models \varphi, \\
(I, M, N)\models \varphi_1 \land \varphi_2  \mbox{ if } (I, M, N)\models \varphi_1 \mbox{ and }(I, M, N)\models \varphi_2, \\
(I, M, N) \models \exists x \varphi  \mbox{ if } (I, M, N) \models \varphi[\alpha/x] \mbox{ for some name } \alpha, \\
(I, M, N) \models \boldK \varphi  \mbox{ if } (J, M, N) \models \varphi \mbox{ for all }J\in M, \\
(I, M, N)\models \boldnot \varphi \mbox{ if } (J, M, N)\not\models \varphi  \mbox{ for some } J\in N.
\end{array}
$$
The symbols $\top$, $\bot$, $\lor$, $\forall$, and $\supset$ are interpreted as usual. 

An {\em MKNF interpretation} $M$ is a nonempty set of first-order interpretations over the universe $\left|I\right|$ for some $I\in M$. In MKNF, a notion called {\em standard name assumption} is imposed to avoid unintended behaviors \cite{Motik:JACM:2010}.
This requires an interpretation to be a Herbrand interpretation with a countably infinite number of additional constants, and the predicate $\approx$ to be a congruence relation.\footnote{The requirement that the predicate $\approx$  be interpreted as a congruence relation overwrites the earlier assumption that $\approx$ is interpreted as equality.} 
Intuitively, given the assumption that each individual in the universe of an interpretation is denoted by a constant and  the countability it implies, the standard name assumption becomes a convenient normalized representation of interpretations since each interpretation is isomorphic to the quotient (w.r.t.~$\approx$) of a Herbrand interpretation and each quotient of a Herbrand interpretation is an interpretation.
In the sequel,  we assume the standard name assumption, and due to this assumption, in definitions we need not explicitly mention the universe associated with the underlying interpretations.

An MKNF interpretation $M$ {\em satisfies} an MKNF formula $\varphi$, written $M\models_\MKNF \varphi$, if $(I, M, M)\models \varphi$ for each $I\in M$. {\em Two-valued MKNF models} are defined as follows.

\begin{definition}
\label{2m} An MKNF interpretation $M$ is an {\em MKNF~model} of an MKNF formula~$\varphi$ if 
\begin{itemize}
 \item[(1)]  $M\models_\MKNF \varphi$, and 
\item [(2)] for all MKNF interpretations $M'$ such that $M'\supset M$, $(I', M', M) \not \models  \varphi$ for every  $I'\in M'$.\end{itemize}
\end{definition}

For example, with the MKNF formula $\varphi = \boldnot b \imply  \bfK a$, it is easy to verify that the MKNF interpretation $M = \{ \{a\}, \{a, b\}\}$ is an MKNF model of $\varphi$.

Following  \cite{KnorrAH11}, a 
{\em three-valued MKNF structure}, $(I, {\M}, {\N} )$, consists of a first-order interpretation, $I$, and two pairs, ${\M} =
\langle M, M_1 \rangle$ and ${\N} = \langle N, N_1 \rangle$, of sets of first-order interpretations, where $M_1 \subseteq M$ and $N_1 \subseteq N$. 
From the two component sets in ${\M}=\langle M,M_1 \rangle$, we can define three truth values for modal $\bfK$-atoms in the following way: $\boldK \varphi$ is true w.r.t.~${\M}=\langle M, M_1\rangle$ if $\varphi$ is true in all interpretations in $M$; it is false if it is false in at least one interpretation in $M_1$; and it is undefined otherwise. For $\bfnot$-atoms, a symmetric treatment w.r.t.~${\N}=\langle N,N_1 \rangle$ is adopted. 
Let $\{\bf t, u, f\}$ be the set of truth values {\em true}, {\em undefined}, and {\em false} with the order ${\bf f} < {\bf u} < {\bf t}$, 
and let the operator $max$ (resp. $min$) choose  the greatest (resp. the least) element with respect to this ordering. Table \ref{three-valued sat} shows three-valued evaluation of MKNF formulas.



A {\em (three-valued) MKNF interpretation pair} $(M,N)$ consists of two MKNF interpretations, $M$ and $N$, with $\emptyset \subset N \subseteq M$.
An MKNF  interpretation pair satisfies an MKNF formula $\varphi$, denoted $(M, N) \models \varphi$, iff $(I, \langle M, N\rangle , \langle M, N\rangle )(\varphi) = {\bf t}$ 
for each $I \in M$. If $M = N$, the MKNF interpretation pair is called {\em total}. 

\begin{table}
{\small
\begin{center}
\caption{Evaluation in three-valued MKNF structure  $(I,{\M},{\N})$
}
\vspace{.05in}
\label{three-valued sat}
\begin{tabular}{lcl} 
\toprule
$(I,{\M},{\N})(P(t_1,\ldots, t_n)) =  
\begin{cases}
    \bf t  & \quad \text{iff}~    {P(t_1,\ldots, t_n)   \mbox{ is true in } I}\\
    \bf f  & \quad \text{iff}~ P(t_1,\ldots, t_n)  \mbox{ is false in } I\\
\end{cases}$\\
\hline

$(I,{\M},{\N})(\neg \varphi) =~~~~~~~~~~~~~~~~  
\begin{cases}
    \bf t  & \quad \text{iff}~ (I,{\M},{\N})(\varphi) = {\bf f}\\
    \bf u  & \quad \text{iff}~ (I,{\M},{\N})(\varphi) = {\bf u}\\
    \bf f  & \quad \text{iff}~ (I,{\M},{\N})(\varphi) = {\bf t}\\
\end{cases}$\\

\hline
$(I,{\M},{\N})(\varphi_1 \wedge \varphi_2)=~~~~~~~~~~\text{min}\{(I,{\M},{\N})(\varphi_1),(I,{\M},{\N})(\varphi_2)\}$ \\

\hline

$(I,{\M},{\N})(\varphi_1\supset \varphi_2) =~~~~~~~~  
\begin{cases}
    \bf t  & \quad  \text{iff}~ (I,{\M},{\N})(\varphi_2) \geq (I,{\M},{\N})(\varphi_1)\\
    \bf f  & \quad  \text{otherwise}\\
\end{cases}$\\

\hline
$(I,{\M}, {\N})(\exists x\!: \varphi)=~~~~~~~~~~~~~~\text{max}\{(I,{\M},{\N})(\varphi[{\alpha}/x]) \,| \,\alpha \mbox{ is a name}\}$\\

\hline

$(I,{\M},{\N})({\boldK} \varphi) =~~~~~~~~~~~~~~~~  
\begin{cases}
    \bf t  & \quad \text{iff}~ (J,\langle M, M_1\rangle , {\N})(\varphi) = {\bf t}~ \text{for all}~ J \in M\\
    \bf f  & \quad \text{iff}~ (J,\langle M, M_1\rangle , {\N})(\varphi) = {\bf f}~ \text{for some}~ J \in M_1\\
    \bf u  & \quad \text{otherwise}\\
\end{cases}$\\

\hline

$(I,{\M},{\N})({\boldnot} \varphi) =~~~~~~~~~~~~~~  
\begin{cases}
    \bf t  & \quad \text{iff}~(J, {\M}, \langle N, N_1\rangle )(\varphi) = {\bf f}~\text{for some}~J \in N_1\\
    \bf f  & \quad \text{iff}~ (J,{\M}, \langle N, N_1\rangle )(\varphi) = {\bf t}~\text{for all}~J \in N\\
    \bf u  & \quad \text{otherwise}\\
\end{cases}$\\

\bottomrule
  \end{tabular}
\end{center}
}
\end{table}

\begin{definition}
\label{3m}
An MKNF interpretation pair $(M, N)$ is a {\em three-valued  MKNF model} of an MKNF formula $\varphi$ if
\begin{itemize}
\item [(a)]
$(M, N)  \models \varphi$, and 
\item [(b)]
for all MKNF interpretation pairs $(M', N')$ with $M \subseteq M'$ and $N \subseteq N'$, where at least one of the inclusions is proper and $M' = N'$ if $M = N$, $\exists I' \in M'$ such that $(I', \langle M', N'\rangle , \langle M, N\rangle )(\varphi)  \neq  {\bf t}$.
\end{itemize}
\end{definition}

Condition (a) checks satisfiability while condition (b), with the evaluation of  $\bfnot$-atoms fixed, constrains the evaluation of modal $\bfK$-atoms to be minimal w.r.t the ordering ${\bf f} < {\bf u} < {\bf t}$ while maximizing falsity. That is, by enlarging $M$ to $M'$ we limit the derivation of $\bfK$-atoms, and by enlarging $N$ to $N'$ we expand on falsity to reduce undefined. Thus,  a three-valued  MKNF model is one for which neither of these is possible under the assumption that $\bfnot$-atoms remain to be evaluated w.r.t.~$(M,N)$. 
If $M = N$, then $(M,M)$ is equivalent to a two-valued MKNF model. The requirement 
$M' = N'$ reduces the definition to one for 
two-valued MKNF models as given in Definition \ref{2m}, which enables Knorr et al. \shortcite{KnorrAH11} to show 
that 
an MKNF interpretation pair $(M,M)$ that is a three-valued MKNF
model of $\varphi$ corresponds to a two-valued MKNF model $M$ of $\varphi$ as defined in \cite{Motik:JACM:2010}.

\begin{example}
Consider the MKNF formula $\varphi = [(\boldnot b \wedge \boldnot a) \supset \boldK a] \wedge [\boldK a \supset \boldK d]
$ and the MKNF interpretation pair $(M,M)$ where $M = \{\{a,d\}, \{a,b,d\}\}$.  
We have $(M,M) \models \{\boldnot b, \neg \boldnot a, \boldK a, \boldK d\}$.  Though $(M,M) \models \varphi$, it violates condition (b) of Definition~\ref{3m},  since the three-valued MKNF structure $(I, \langle { M',M'}\rangle, \langle M,M\rangle)$, where $M' = \{\emptyset, \{a,d\},  \{a,b,d\}\}$ and thus $M \subset M'$, evaluates $[\boldnot b, \boldnot a, \boldK a, \boldK d]$ to $[{\bf t}, {\bf f}, { {\bf f}}, {\bf f}]$, respectively,  independent of $I$. It follows that $(I, \langle { M',M'}\rangle, \langle M,M\rangle)$ evaluates 
$\varphi$ to ${\bf t}$, according to Table \ref{three-valued sat}. 
\end{example}

MKNF interpretation pairs can be compared by an {\em order of knowledge}. 
Let $(M_1,N_1)$ and $(M_2,N_2)$ be MKNF interpretation pairs. $(M_1,N_1) \succeq_k (M_2,N_2)$ iff $M_1 \subseteq M_2$ and $N_1 \supseteq N_2$.
A  three-valued MKNF model $(M,N)$  of an MKNF formula $\varphi$ 
 is called a {\em well-founded MKNF model} of $\varphi$ if 
$(M_1,N_1) \succeq_k (M,N)$ for all three-valued MKNF models $(M_1,N_1)$ of $\varphi$.


\subsection{Hybrid MKNF knowledge bases}
\label{4.2}

The critical issue of how to combine open and closed world reasoning is addressed  in \cite{Motik:JACM:2010} by  seamlessly integrating rules with DLs.  
A hybrid MKNF knowledge base $\K = (\cal O,\P) $ consists of a 
decidable description logic (DL) knowledge base $\cal O$, translatable into first-order logic and a  rule base $\P$, which is 
a finite set of rules with modal atoms.
The original work on hybrid MKNF knowledge bases \cite{DBLP:conf/ijcai/MotikR07,Motik:JACM:2010} defines a two-valued semantics for such knowledge bases with disjunctive rules. In this paper, following \cite{KnorrAH11},
our focus is on nondisjunctive rules as presented in \cite{DBLP:conf/ijcai/MotikR07}.

An MKNF rule (or simply a {\em rule}) $r$ is of the form: $
{\boldK} H \leftarrow \boldK A_1,\ldots, \boldK A_m, \boldnot B_1,$ $ \ldots, \boldnot B_n$,
where ${ H}, A_i,$ and $B_j$ are function-free first-order atoms. Given a
rule $r$, we let 
 $\head(r) =\boldK H$, $\body^+ (r) = \{\boldK A_i\,|\, i = 1 .. m\}$, and $\body^-(r) = \{B_i\,|\,i= 1.. n\}$. 
A rule is {\em positive} if it contains no $\bfnot$-atoms.  When all rules in $\P$ are positive,  $\K = (\cal O,\P) $ is called {\em positive}.

For the interpretation of a hybrid MKNF knowledge base  $\K = (\cal O, \P)$ in the logic of MKNF, a transformation $\pi (\K) = \boldK \pi(\cal O) \wedge \pi(\P)$ is performed to transform $\cal O$ into a first-order formula and rules $r \in \P$ into a conjunction of  first-order implications to make each of them coincide syntactically with an MKNF formula. More precisely, 
$$
\begin{array}{ll}
\pi(r) = \forall \vec{x}\!:  ({\boldK} H \subset \boldK A_1 \wedge \ldots \wedge \boldK A_m \wedge \boldnot B_1 \wedge \ldots\wedge  \boldnot B_n)\\
\pi({\cal P})= \bigwedge_{r \in {\cal P}} \pi(r), ~~
\pi({\cal K}) = \bfK \pi({\cal O}) \wedge \pi({\cal P})
\end{array}
$$
where $\vec{x}$ is the vector of free variables in $r$. 

Under the additional assumption of DL-safety a first-order rule base is semantically equivalent to a finite ground rule base, in terms of two-valued MKNF models \cite{Motik:JACM:2010} as well as in terms of three-valued MKNF models \cite{KnorrAH11}; hence decidability is guaranteed. 
 Given a hybrid MKNF knowledge base $\K = (\cal O,\P)$, 
 a rule $r$ in $\P$  is said to be {\em DL-safe} if every variable in $r$ occurs in at least one $\bfK$-atom  in the body of $r$ whose predicate symbol does not appear in $\cal O$,\footnote{Such a modal $\bfK$-atom is called a {\em non-DL-atom} in \cite{KnorrAH11,Motik:JACM:2010}.} and $\K$ is DL-safe if all rules in $\P$ are DL-safe. 
In this paper, we assume that a given rule base is always DL-safe, and for convenience, when we write $\P$ we assume it is already grounded. 

Given a hybrid MKNF knowledge base $\K = (\cal O, \P)$, let 
$\KA(\K)$ be the set of all (ground) $\bfK$-atoms $\boldK \phi$
such that either $\boldK \phi$ occurs in $\P$ or $\boldnot\phi$ occurs in $\P$. 
We generalize  the notion of partition \cite{KnorrAH11} from consistent pairs to all pairs: A {\em partition} of $\KA(\K)$ is a pair $(T, P)$ such that $T, P \subseteq \KA (\K)$;  if $T \subseteq P$, then $(T,P)$ is said to be consistent, otherwise it is inconsistent.
A partition of the form $(E,E)$ is said to be {\em exact}.

 Intuitively, for a partition $(T,P)$, $T$ contains {\em true} modal $\bfK$-atoms and $P$ contains {\em possibly true} modal $\bfK$-atoms. Thus, the complement of $P$ is the set of {\em false} modal $\bfK$-atoms and $P\backslash T$ the set of {\em undefined} modal $\bfK$-atoms. 

Partitions are closely related to MKNF interpretation pairs. 
It is shown in \cite{KnorrAH11,LY17} that an MKNF interpretation pair $(M, N)$ induces a consistent partition $(T, P)$ such that for any modal $\bfK$-atom $\boldK \xi \in \KA(\K)$,

\begin{enumerate}
  \item $\boldK \xi \in T$ iff $\forall I \in M, (I, \langle M, N\rangle , \langle M, N\rangle )(\boldK \xi) = {\bf t}$,
  \item $\boldK \xi \not \in P$ iff $\forall I \in M, (I, \langle M, N\rangle , \langle M, N\rangle )(\boldK \xi) = {\bf f}$, and 
  \item $\boldK \xi \in P \backslash T$ iff  $\forall I \in M, (I, \langle M, N\rangle , \langle M, N\rangle )(\boldK \xi) = {\bf u}$.
\end{enumerate}

Given a set of first-order atoms $S$, we define the corresponding set of modal $\bfK$-atoms as: $\boldK (S) = \{ \boldK \phi\,|\, \phi \in S\}$.



Let $S$ be a subset of $\KA(\K)$. 
The {\em objective knowledge} of $S$ relevant to $\K$ is the set of first-order formulas ${\OB}_{{\cal O},S} = \{\pi({\cal O})\} \cup \{\xi~|~ {\boldK \xi} \in S\}$.  

\begin{example}
Consider a hybrid MKNF knowledge base $\K=({\cal O},{\P})$, where ${\cal O} = a \wedge (b \supset  c) \wedge \neg f$ and
$\P$ is
$$
\begin{array}{ll} 
{\boldK} b \leftarrow {\boldK} a.  ~~~
{\boldK} d \leftarrow {\boldK}c, {\boldnot} e. ~~~{\boldK} e \leftarrow {\boldnot}d. ~~~{\boldK}f \leftarrow {\boldnot} b.
\end{array}
$$
Reasoning with $\K$ can be understood as follows:
since $\boldK {\cal O}$ implies $\boldK a$, by 
 the first rule we derive $\boldK b$; then due to $b \supset c$ in $\cal O$ we derive
 $\boldK c$. Thus its occurrence in the body of the  second rule is true and can be ignored. For the ${\bf K}$-atoms $\boldK d$ and $\boldK e$ appearing in the two rules in the middle,  without preferring one over the other, both can be undefined.  Because both $\boldnot b$ and  $\boldK f$ are false (the latter is due to $\neg f$ in $\cal O$), the last rule is also satisfied. 
Now
consider an MKNF interpretation pair $(M,N)=(\{I\,|\, I \models {\cal O} \wedge b \},
\{I\,|\,I\models {\cal O} \wedge b \wedge d \wedge e\})$, which corresponds to partition $(T,P) = (\{\boldK a, \boldK b, 
\boldK c\}, \{\boldK a, \boldK b, \boldK c, \boldK d, \boldK e\})$. 
For instance, we have that, for all $I \in M$, $(I, \langle M,N\rangle, \langle M,N \rangle)({\boldK} a)=\bf t$ and 
$(I, \langle M,N\rangle, \langle M,N \rangle) ({\boldK} d)=\bf u$.  The interpretation pair $(M,N)$ is a three-valued MKNF model of $\K$; in fact, it is the well-founded MKNF model of $\K$.
\end{example}

It is known that in general the well-founded MKNF model may not exist.

\begin{example} {\rm \cite{LY17}}
\label{no_wfm}
Let us consider $\K = (\cal O,\P)$, where 
${\cal O} = (a \supset  h) \wedge (b \supset \neg h)$ and 
$\P$ consists of 
$$
\begin{array}{ll} 
\boldK a \leftarrow {\boldnot} b. ~~ \boldK b  \leftarrow \boldnot a.
\end{array}
$$
Consider two partitions, $(\{{\boldK} a \}, \{{\boldK} a\})$ and $(\{{\boldK} b\}, \{{\boldK} b\})$. The corresponding  MKNF  interpretation pairs  turn out to be two-valued MKNF models of $\K$.
For example, for the former the interpretation pair is $(M, M)$, where $M = \{ \{a ,h\} \}$. 
Since these two-valued MKNF models are not comparable w.r.t.~undefinedness and there are no other 
three-valued MKNF models of $\K$, it follows that 
no well-founded MKNF model for $\K$ exists. 
\end{example}

\section{Approximators for Hybrid MKNF Knowledge Bases}
\label{approximator-for-mknf}

In this section, we first show that the alternating fixpoint operator defined by Knorr et al. \shortcite{KnorrAH11} can be recast as an approximator of AFT, and therefore can be applied to characterize all three-valued MKNF models automatically and naturally.  We show that this approximator is a strong approximator. Since this approximator is not symmetric, we have discovered a strong approximator for an important application without the assumption of symmetry. Being strong 
guarantees that all consistent stable fixpoints are preserved. At the end,  we show how stable fixpoints of this approximator serve as the candidates for three-valued MKNF models by a simple consistency test. 


Throughout this section, the underlying complete lattice  is $(2^{\KA(\K)}, \subseteq)$ and the induced product bilattice is 
$(2^{\KA(\K)})^2$. 

We define an operator on $2^{\KA(\K)}$, which is 
to be approximated by our approximators introduced shortly.
\begin{definition}
Let $\K = (\cal O,\P)$ be a hybrid MKNF knowledge base.
We define an operator ${\cal T}_{\K}$ on $2^{\KA(\K)}$ as follows:
given $I \subseteq {\KA(\K)}$,
$$
\begin{array}{ll}
{\cal T}_{\K}(I) = \{\boldK a\in {\KA}({\K}) \mid {\OB}_{{\cal O}, I}\models a\}  \, \cup \\
~~~~~~~~~~~~~~~~\{\head(r) \mid  r \in {\P}:  \, \body^+(r) \subseteq I, \, \boldK(\body^-(r)) \cap I = \emptyset\}
\end{array}$$
\end{definition}

If $\K$ is a positive hybrid MKNF knowledge base, the operator ${\cal T}_{\K}$ is monotone and has a least fixpoint. If in addition $\cal O$ is an empty DL knowledge base, then ${\cal T}_{\K}$ is essentially the familiar {\em immediate consequence operator} of \cite{EmdenK76}.

Knorr et al. \shortcite{KnorrAH11} defined two kinds of transforms with consistent partitions. For the purpose of this paper, let us  allow arbitrary partitions. 
\begin{definition}
Let $\K = (\cal O, \P)$ be a hybrid MKNF knowledge base and $S \in 2^{\KA(\K)}$.
Define two forms of reduct: 
$$
\begin{array}{ll}
{\K }/ S = ({\cO},{\P'}), \mbox{ where } \\
~~~~~~~~~~~~{\P'} = \{  \bfK a \leftarrow bd^+(r)~|~r \in {\P}: \head(r) = \bfK a, \bfK ( bd^-(r)) \subseteq {\KA(\K)} \setminus S\}
\\
{\K }// S = ({\cO}, {\P''}), \mbox{ where }\\
~~~~~~~~~~~~{ \P''} = \{  \bfK a \leftarrow bd^+(r)~|~r \in {\P}: \head(r) = \bfK a, ~
\bfK ( bd^-(r)) \subseteq {\KA(\K)} \setminus S, 
{\OB}_{{\cal O}, S} \not \models \neg a \}
\end{array}
$$
We call ${\K} /S$ {\em MKNF transform} and ${\K} //S$ {\em MKNF-coherent transform. }
\end{definition}

Since in both cases of ${\K/}S$ and ${\K}//S$
the resulting rule base is positive, a least fixpoint in each case exists. Let us define $\Gamma_{\K} (S) = \lfp({\cal T}_{{\K} /S})$ and $\Gamma'_{\K} (S) = \lfp({\cal T}_{{\K}//S})$.  Then, we can construct two sequences $\boldP_i$ and $\boldN_i$ as follows:
$$
\begin{array}{ll}
\label{alter}
\boldP_0 = \emptyset, \ldots,  \boldP_{n+1} = \Gamma_{\K}(\boldN_n), \ldots, \boldP_\Omega = \bigcup \boldP_i \\
 \boldN_0 = {\KA({\K})}, \ldots,  \boldN_{n+1} = \Gamma'_{\K}(\boldP_n), \ldots, \boldN_\Omega = \bigcap \boldN_i
\end{array}
$$

Intuitively, starting from  $\boldP_0$  where no modal $\bfK$-atoms are known to be true and $\boldN_0$ where all modal $\bfK$-atoms are possibly true, $\boldP_{i+1} $ computes the true modal $\bfK$-atoms given the set of possibly true modal $\bfK$-atoms in $\boldN_i$, and $\boldN_{i+1}$ computes the set of possibly true modal $\bfK$-atoms given the $\bfK$-atoms are known be true in $\boldP_i$.
Now let us place this construction under AFT by formulating an approximator.

\begin{definition}
\label{Phi}
Let $\K = (\cal O, \P)$ be a hybrid MKNF knowledge base.
We define an operator $\Phi_{\K}$ on $(2^{\KA(\K)})^2$ as follows: 
$\Phi_{\K}(T,P) = (\Phi_{\K} (T,P)_1 , \Phi_{\K}(T,P)_2)$,
where
$$
\begin{array}{ll}
\Phi_{\K} (T,P)_1 =  \{\boldK a\in {\KA}({\K}) \mid {\OB}_{{\cal O}, T}\models a\} \, \cup \, 
\\~~~~~~~~~~~~~~~~~~~~~~~~~~~~~~~\{ \head(r) \mid  r \in {\P}:  \, \body^+(r) \subseteq T, \, \boldK(\body^-(r)) \cap P = \emptyset\}
\\
\Phi_{\K} (T,P)_2 =  \{\boldK a\in {\KA}({\K}) \mid {\OB}_{{\cal O}, P}\models a\} \,\,  \cup \,
\\~~~~~~~~~~~~~~~~~~~~~~~~~~~~~~~\{ \head(r) \mid  r \in {\P}:  \, \head(r) = \bfK a , \,
{\OB}_{{\cal O}, T} \not \models \neg a,  \, \body^+(r) \subseteq P, \,\\ ~~~~~~~~~~~~~~~~~~~~~~~~~~~~~~~~~~~~~~~~~~~~~\boldK(\body^-(r)) \cap T = \emptyset\}
\end{array}
$$
\end{definition}

Intuitively, given a partition $(T,P)$, the operator $\Phi_{\K}(\cdot,P)_1$, with $P$ fixed, computes the set of true modal $\bfK$-atoms w.r.t.~$(T,P)$ 
and operator $\Phi_{\K}(T,\cdot)_2$, with $T$ fixed, computes the set of modal $\bfK$-atoms that are possibly true w.r.t.~$(T,P)$. 

Note that the least fixpoint of operator $\Phi_{\K}(\cdot,P)_1$ corresponds to an element in the sequence $\boldP_i$, i.e.,  if $P$ in $\Phi_{\K}(\cdot,P)_1$ is $\boldN_n$, then $\lfp(\Phi_{\K}(\cdot,P)_1)$ is $\boldP_{n+1} = \Gamma_{\K}(\boldN_n)$. Similarly for operator $\Phi_{\K}(T,\cdot)_2$.  
In this way, the $\Phi_{\K}$ operator can be seen as a reformulation of the corresponding alternating fixpoint operator; namely, 
$\Phi_{\K}(\cdot,P)_1$ simulates operator ${\cal T}_{{\K}/P}$ and $\Phi_{\K}(T,\cdot)_2$ simulates operator ${\cal T}_{{\K}//T}$.

\begin{proposition}
\label{approximator0}
$\Phi_{\K}$ is an approximator for ${\cal T}_{\K}$.
\end{proposition}
\begin{proof}
Let us check $\subseteq_p$-monotonicity of $\Phi_{\K}$.
 Let $(T_1,P_1) \subseteq_p (T_2,P_2)$. From $T_1 \subseteq T_2$ and $P_2 \subseteq P_1$,  it is easy to verify that 
 $\Phi_{\K}(T_1,P_1)_1 \subseteq \Phi_{\K}(T_2,P_2)_1$. For 
 $\Phi_{\K}(T_2,P_2)_2 \subseteq \Phi_{\K}(T_1,P_1)_2$, 
 note that $\Phi_{\K}(\cdot, \cdot)_2$ is defined 
in terms of two subsets. For the first subset, 
since
$P_2 \subseteq P_1$, the set defined w.r.t. $P_2$ 
is a subset of the set defined w.r.t. $P_1$,  i.e., $\{\boldK a\in {\KA}({\K}) \mid {\OB}_{{\cal O}, P_2}\models a\}$ is a subset of  $\{\boldK a\in {\KA}({\K}) \mid {\OB}_{{\cal O}, P_1}\models a\}$.
For the second subset, along with $T_1 \subseteq T_2$, the set defined w.r.t. $(T_2,P_2)$ is a subset of 
the set defined  w.r.t. $(T_1,P_1)$.
Thus  $\Phi_{\K}(T_1,P_1) \subseteq_p \Phi_{\K}(T_2,P_2)$. 
Furthermore, $\Phi_{\K}$ approximates ${\cal T}_{\K}$, since by definition $\Phi_{\K}(I, I)_1 \supseteq \Phi_{\K}(I, I)_2$, and it follows that 
whenever $\Phi_{\K}(I, I)$ is consistent, 
$\Phi_{\K}(I, I) = ({\cal T}_{\K} (I), {\cal T}_{\K} (I))$.
\end{proof}



\begin{example}
Consider a hybrid MKNF knowledge base $\K=({\cal O},{\P})$, where ${\cal O} = c \wedge (e \supset  \neg r)$ and
$\P$ consists of
$$
\begin{array}{ll} 
{\boldK} r \leftarrow {\boldK} c, {\boldK} i,  {\boldnot} o, {\boldnot} l.~~~~~~
{\boldK} e \leftarrow. ~~~~~~{\boldK} i \leftarrow.~~~~~~~~~~
\end{array}
$$
One can derive that, for the exact pair $(T,T)=(\{{\boldK} c, {\boldK} i, {\boldK} e\}, \{{\boldK} c, {\boldK} i, {\boldK} e\})$,
$\Phi_{\K}(T,T)=(\{{\boldK} c, {\boldK} i, {\boldK} e, {\boldK} r, {\boldK} o, {\boldK} l\}, \{{\boldK} c, {\boldK} i, {\boldK} e\})$. Operator $\Phi_{\K}$ maps the exact pair $(T,T)$ to an
inconsistent one, and it is therefore not a symmetric approximator. Note that 
the least stable fixpoint of 
$\Phi_{\K}$ is just the mapped inconsistent pair. It is interesting to see the information revealed in this stable fixpoint \-- while it is inconsistent, it provides consistent information on $\bfK$-atoms, $\boldK c, \boldK i$, and $\boldK e$.  
\end{example}

\comment{
Though not symmetric in general, approximator $\Phi_{\K}$ is symmetric on certain pairs of interest.
\begin{lemma}
Let $\K = (\cal O, \P)$ be a hybrid MKNF knowledge base and 
$(T,P) \in (2^{\KA(\K)})^2$ such that  ... 
 {\color{red} (Don't know what condition leads to the next conclusion.)}
Then $\Phi_{\K} (T,P)_2 = \Phi_{\K} (P,T)_1$. 
\label{symmetry0}
\end{lemma}

\begin{proof}
The proof uses  Def. \ref{Phi}.  Assume $(T,P) \in (2^{\KA(\K)})^2$ such that  ${\OB}_{{\cal O}, T}$ is satisfiable. 
 Let $\bfK \, a \in  \Phi_{\K}(T,P)_2$. 
If 
${\OB}_{{\cal O}, T} \not \models \neg a$,  then the definition of  $\Phi_{\K}(T,P)_2$ reduces to the definition of $\Phi_{\K}(P,T)_1$ with regard to how $\bfK \, a$ is derived, and therefore 
$\bfK \, a \in \Phi_{\K}(P,T)_1$. 
Otherwise ${\OB}_{{\cal O}, T} \models \neg a$. Then in the definition of $\Phi_{\K}(T,P)_2$,
$\bfK \, a$ cannot be derived by rules and we must have
 ${\OB}_{{\cal O}, P} \models a$, and thus $\bfK \, a \in  \Phi_{\K}(P,T)_1$.  

Now assume $\bfK \, a \in  \Phi_{\K}(P,T)_1$. We can simulate the derivations in the construction of  $\Phi_{\K}(P,T)_1$ in order to construct $\Phi_{\K}(T,P)_2$.
Consider any derivation 
in  $\Phi_{\K}(P,T)_1$, where 
 $\bfK\,a$ is derived by a rule $r \in {\cal P}$ such that  $\head(r) = \bfK\, a$, 
 $
 \, \body^+(r) \subseteq P$ and $\boldK(\body^-(r)) \cap T = \emptyset$.
If  ${\OB}_{{\cal O}, T} \not \models \neg a$, 
 then exactly the same rule applies to derive $\bfK \, a$ in $\Phi_{\K}(T,P)_2$.
Otherwise ${\OB}_{{\cal O}, T} \models \neg a$ in which case if 
${\OB}_{{\cal O}, P} \models a$, then $\bfK\,a\in \Phi_{\K}(T,P)_2$, but if ${\OB}_{{\cal O}, P} \not \models a$ and  with ${\OB}_{{\cal O}, T} \models \neg a$,
no rules can be applied to derive $\bfK\,a$ in $\Phi_{\K}(T,P)_2$, i.e., 
$\bfK \, a \not \in \Phi_{\K}(T,P)_2$. {\color{red} (This is the trouble, I couldn't come with some condition such that this would not happen.)}
\end{proof}
}
\comment{
{\color{blue}
\begin{proof}
{
Here, we only consider consistent stable fixpoints. Given $(T,P)$ a stable fixpoint of $\Phi_{\K}$, then $T \subseteq P$, $T=\lfp(\Phi_{\K}(\cdot,P))_1$ and 
$P=\lfp(\Phi_{\K}(T,\cdot))_1$. We show that $\Phi_{\K} (T,P)_2 = \Phi_{\K} (P,T)_1$. 

$(\Rightarrow)$ Given $\bfK a \in \Phi_{\K} (T,P)_2$, either (i) ${\OB}_{{\cal O}, P} \models a$, or (ii)${\OB}_{{\cal O}, T} \models \neg a$ and there exists a rule $r$ s.t. $\bfK a$ can be derived.
For case (i), surely $\bfK a \in \Phi_{\K} (P,T)_1$ by Def. \ref{approximator0}. As to case (ii), again $\bfK a$ can be derived by such a rule $r$, therefore, $\bfK a \in \Phi_{\K} (P,T)_1$.

$(\Leftarrow)$Given $\bfK a \in \Phi_{\K} (P,T)_1$, also we have either (i) ${\OB}_{{\cal O}, P} \models a$, or (ii)there exists a rule $r$ s.t. $\bfK a$ can be derived.
For case (i), then $\bfK a \in \Phi_{\K} (T,P)_2$ by Definition\ref{approximator0}. As to case (ii), if ${\OB}_{{\cal O}, T} \not \models \neg a$, it naturally leads to that $\bfK a$ can be derived by rule $r$, 
and $\bfK a \in \Phi_{\K} (T,P)_2$. Otherwise, if ${\OB}_{{\cal O}, T} \models \neg a$, then ${\OB}_{{\cal O}, P} \models \neg a$ by $T \subseteq P$, on the otherside, since $P=\lfp(\Phi_{\K}(T,\cdot))_1$, we have 
$P={\KA(\K)}$, in turn, it also holds that $\Phi_{\K} (P,T)_1=\KA(\K)$, then $\Phi_{\K} (T,P)_2 = \Phi_{\K} (P,T)_1$.}
\end{proof}

}
}

Our next goal is to show that the operator $\Phi_{\K}$ is a strong approximator under a mild condition.
First, let us introduce some notations. 
Recall that we use the notation $(D_1(P), D_1(T)) = \lfp(\Phi_{\K}(\cdot, P)_1, \lfp(\Phi_{\K}(T,\cdot)_2)$,
where $\Phi_{\K}(\cdot, P)_1$ is defined on $[\emptyset, P]$ and $\Phi_{\K}(T,\cdot)_2$ is defined on $[T,\KA(\K)]$, and  the notation $(C_1(P), C_2(T)) = \lfp(\Phi_{\K}(\cdot, P)_1, \lfp(\Phi_{\K}(T,\cdot)_2)$, where both 
$\Phi_{\K}(\cdot, P)_1$ and $\Phi_{\K}(T,\cdot)_2$ are operators on $\KA(\K)$. We now give notations to refer to intermediate results in a least fixpoint construction (we define them here for $D_1 (P)$ and $C_2 (T)$; others are similar):
$$\begin{array}{ll}
D_1^{\uparrow 0} (P)= \emptyset   ~~~~~~~~~~~~\,~~~~~~~~~~\,~~~~~~~~~~~~~~~~~C_2^{\uparrow 0} (T)= \emptyset\\
D_1^{\uparrow k+1} (P)  = \Phi_{\K} ( D_1^{\uparrow k}(P), P)_1  ~~~~~~~~~~~ C_2^{\uparrow k+1} (T) = \Phi_{\K} (T, C_2^{\uparrow k} (T))_2  ~~~\mbox{for all } k \geq 0 \\
\end{array}
$$

\begin{proposition}
\label{strong0}
Let $\K = (\cal O, \P)$ be a hybrid MKNF knowledge base and $(T,P)$ be a consistent stable fixpoint of $\Phi_{\K}$ such that 
${\OB}_{{\cal O}, T}$ is satisfiable. Then  $\Phi_{\K}$ is a strong approximator for $(T,P)$.
\end{proposition}


\begin{proof}
Let $(T,P)$ be a consistent stable fixpoint of $\Phi_{\K}$ such that 
${\OB}_{{\cal O}, T}$ is satisfiable.
We show that 
$\Phi_{\K}$ is strong for $(T,P)$. 
We need to show $(C_1(P), C_2(T))= (D_1(P), D_2(T))$.
That $C_1(P) = D_1(P)$ is immediate since the monotonicity of the projection operators implies that the construction of the least fixpoint in both cases starts with the same least element, $\emptyset$, and is carried out in tandem by the same mapping, and therefore terminates at the same fixpoint.  

That $C_2(T) \subseteq D_2(T)$ is also easy to show by induction. The construction of the least fixpoint by $C_2(T)$ starts from $\emptyset$ 
and the one by $D_2(T)$ starts from $T$. So for the base case,  $C_2^{\uparrow 0}(T) \subseteq  D_2^{\uparrow 0}(T)$. Then, one can verify by definition that 
for any (fixed) $k \geq 0$, by the monotonicity of the projection operators on their respective domains, that $C_2^{\uparrow k}(T) \subseteq D_2^{\uparrow k}(T)$ implies  $C_2^{\uparrow k+1}(T) \subseteq D_2^{\uparrow k+1}(T)$. 

To show $D_2(T) \subseteq C_2(T)$, we first prove by induction that
$D_1(P)  \subseteq C_2(T)$. Since $D_1(P)  = C_1(P) = T$, this is to show $T \subseteq C_2(T)$.
The base case  is immediate since both least fixpoint constructions start with the same least element $\emptyset$. Assume 
$D_1^{\uparrow k}(P) \subseteq C_2^{\uparrow k} (T)$ for any (fixed) $k\geq 0$, and we show it for $k+1$. By definition, a new $\bfK$-atom $\bfK \, a$ is added to $D_1^{k+1}(P)$ because (i)  ${\OB}_{{\cal O}, D_1^{\uparrow k}(P) }\models a$, or (ii) there is a rule
$r \in {\cal P}$ with $\head(r) = \bfK \, a$ such that 
$\body^+(r) \subseteq D_1^{\uparrow k}(P)$ and $ \boldK(\body^-(r)) \cap P = \emptyset$. 
If case (i) applies, by induction hypothesis (I.H.), it follows ${\OB}_{{\cal O}, C_2^{\uparrow k}(T) }\models a$, and thus $\bfK \, a \in C_2^{\uparrow k+1}(T)$. Otherwise, $\bfK \, a$ is derived only by rules as in case (ii). Note that since $D_1^{\uparrow k}(P) \subseteq T$, case (ii) implies ${\OB}_{{\cal O}, T} \models a$.  If ${\OB}_{{\cal O}, T} \models \neg a$, then ${\OB}_{{\cal O}, T}$ is unsatisfiable, violating the assumption that ${\OB}_{{\cal O}, T}$ is satisfiable.
Thus, we must have ${\OB}_{{\cal O}, T} \not \models \neg a$; then the same rule applied in case (ii) above applies in the construction of $C_2^{\uparrow k+1}$, since the condition $(\body^+(r) \subseteq D_1^{\uparrow k}(P)$ and $ \boldK(\body^-(r)) \cap P = \emptyset)$ becomes $(\body^+(r) \subseteq C_2^{\uparrow k}(T)$ and $ \boldK(\body^-(r)) \cap T = \emptyset)$), which holds 
by I.H. and the fact that $T \subseteq P$. Thus, $D_1(P)  ~(=T) \subseteq C_2(T)$.

Once we obtain $T \subseteq C_2(T)$, we are ready to conclude that
$\lfp(\Phi_{\K}(T,\cdot)_2)$ with the operator $\Phi_{\K}(T,\cdot)_2$ defined 
on domain $[T,\KA(\K)]$ is 
a subset of $\lfp(\Phi_{\K}(T, \cdot)_2)$ with the operator $\Phi_{\K}(T,\cdot)_2$ defined
on domain $\KA(\K)$. This is because the construction of the former least fixpoint starts from the least element $T$ of the
 domain $[T,\KA(\K)]$, and the construction of the latter is guaranteed to reach a set $T' \supseteq T$, and by induction on both sequences in parallel, 
we have $D_2(T) \subseteq C_2(T)$.
\end{proof}

\comment{
\begin{proof}
We show that $\Phi_{\cal K}$ is strong for $(T,P)$.  Assume $(T,P)$ is a coherent, consistent stable fixpoint computed by 
$(T,P) = \lfp(\Phi_{\K}(\cdot, P)_1, \lfp(\Phi_{\K}(T,\cdot)_2)$, where
$\Phi_{\K}(\cdot, P)_1$ is defined on $[\emptyset, \KA(\K)]$ and $\Phi_{\K}(T,\cdot)_2$ is defined on $[T,\KA(\K)]$. We show 
$(T,P) = \lfp(\Phi_{\K}(\cdot, P)_1, \lfp(\Phi_{\K}(T,\cdot)_2)$, where both 
$\Phi_{\K}(\cdot, P)_1$ and $\Phi_{\K}(T,\cdot)_2$ defined on $\KA(\K)$.

\vspace{.5in}

Let $(T',P') = \lfp(\Phi_{\K}(\cdot, P')_1, \lfp(\Phi_{\K}(T',\cdot)_2)$, where both 
$\Phi_{\K}(\cdot, P')_1$ and $\Phi_{\K}(T,\cdot)_2$ are defined on $KA(\K)$. If we can show $P = P'$, then $(T,P) = (T', P')$. That 
$P' \subseteq P$ is automatic by monotonicity of operator 
Since $\Phi_{\K}(T',\cdot)_2$ is monotone on $\KA(\K)$, we have $\emptyset \subseteq T$ implies 
$\Phi_{\K}(T,\emptyset)_2 \leq \Phi_{\K}(T,T)_2$, and then by an easy induction on the sequence of the least fixpoint point construction, we have $P' \subseteqq P$. 

\end{proof}
}

Stable fixpoints of the operator $\Phi_{\K}$ can be related to three-valued MKNF models of $\K$ in the following way.
\begin{theorem}
\label{key-result}
Let $\K = (\cal O,\P)$ be a hybrid MKNF knowledge base and $(T,P)$ be a partition. Let further $( M , N ) = (\{ I\, |\, I \models {\OB}_{{\cal O} , T}\},\{I\,|\,I \models {\OB}_{{\cal O},P}\})$.
Then, 
$(M,N)$ is a three-valued MKNF model of $\K$ iff $(T,P)$ is a consistent stable fixpoint of $\Phi_{\K}$  and 
${\OB}_{{\cal O}, \lfp(\Phi_{\K}(\cdot,T)_1)}$ is satisfiable. 
\end{theorem}


Note that in the formulation of approximator $\Phi_{\K}$, stable fixpoints are partitions that provide candidate interpretation pairs for three-valued MKNF models. The extra condition that ${\OB}_{{\cal O}, \lfp(\Phi_{\K}(\cdot,T)_1)}$ is satisfiable means that even if we make all $\bfnot$-atoms $\boldnot \phi$ true when $\phi \not \in T$, in the construction of $\lfp(\Phi_{\K}(\cdot,T)_1)$, it still does not cause contradiction with the DL knowledge base $\cal O$. This provides a key insight in the semantics of hybrid MKNF knowledge bases.

Notice also that this theorem provides a naive method, based on guess-and-verify, to compute three-valued MKNF models of a given hybrid MKNF knowledge base $\K$ \-- guess a consistent partition $(T,P)$ of $\KA (\K)$ and check whether $(T,P)$  is a stable fixpoint of $\Phi_{\K}$ and whether ${\OB}_{{\cal O}, \lfp(\Phi_{\K}(\cdot,T)_1)}$ is satisfiable. Observe that the complexity of checking for one guessed partition is polynomial if the underlying DL is polynomial.  

\begin{proof}
($\Leftarrow$)
Assume that $(T,P)$ is a consistent stable fixpoint of $\Phi_{\K}$ and ${\OB}_{{\cal O}, \lfp(\Phi_{\K}(\cdot,T)_1)}$
is satisfiable. Let $\theta(x)$ denote $\lfp(\Phi_{\K}(\cdot,x)_1)$, given $ x \subseteq \KA({\K})$. Thus,
 ${\OB}_{{\cal O}, \lfp(\Phi_{\K}(\cdot,T)_1)}$ is often written as ${\OB}_{{\cal O}, \theta(T)}$.
By the definition of the operator $\Phi_{\K}$  (cf. Definition \ref{Phi}),
it can be seen that 
$\Phi_{\K} (\cdot, T)_1$ coincides with $\Phi_{\K}(T, \cdot)_2$ except for the extra condition ${\OB}_{{\cal O}, T} \not \models \neg a$ (where $\head(r) = \boldK a$ for some $r \in \P$) in the definition of the latter. Note that the operator $\Phi_{\K}(T, \cdot)_2$ is defined on $\KA(\K)$.
It then follows 
$\lfp(\Phi_{\K}(T, \cdot)_2) \subseteq \theta(T)$. Since $(T,P)$ is a stable fixpoint of $\Phi_{\K}$, $\lfp(\Phi_{\K}(T, \cdot)_2)  = P$ and thus $P \subseteq  \theta(T)$. 
Then, that ${\OB}_{{\cal O}, \theta(T)}$ is satisfiable implies that ${\OB}_{{\cal O}, P}$ is satisfiable, and because $(T,P)$ is  consistent and thus $T \subseteq P$, 
${\OB}_{{\cal O}, T}$ is satisafiable as well. 
It follows that the pair $$( M , N ) = (\{ I\, |\, I \models {\OB}_{{\cal O} , T}\},\{I\,|\,I \models {\OB}_{{\cal O},P}\})$$
is an MKNF interpretation pair because 
$\emptyset \subset N \subseteq M$. As shown by \cite{KnorrAH11},
 for any $\boldK \xi \in \KA(\K)$, ${\boldK} \xi \in T$ iff ${\boldK} \xi$ evaluates to $\bf t$ (under $(M,N)$),
${\boldK} \xi \not \in P$ iff ${\boldK} \xi$  evaluates to $\bf f$,  and otherwise ${\boldK} \xi$ evaluates to $\bf u$
 (also see the review of this property in Section \ref{4.2}, or 
 \cite{DBLP:conf/ruleml/LiuY19} for more details).

We now show that $(M,N)$ is a three-valued MKNF model of $\K$.  First we show that $(M,N)$ satisfies $\pi({\K})$.
Since ${\OB}_{{\cal O},T} = \{\pi({\cal O})\} \cup \{\xi~|~ {\boldK \xi} \in T\}$ 
and ${\OB}_{{\cal O},P} = \{\pi({\cal O})\} \cup \{\xi~|~ {\boldK \xi} \in P\}$, 
it follows $(M,N) \models \boldK \pi(\cal O)$.  Now consider any rule $r \in \P$. Let $\head(r) = \boldK a$. 
By  the definition of $\Phi_{\K}(T,P)_1$,  if $\body^+(r) \subseteq T$ and $\bfK(\body^-(r)) \cap P = \emptyset$, then $\boldK a\in T$; for $\Phi_{\K}(T,P)_2$, 
if $\body^+(r) \subseteq P$, $\bfK(\body^-(r)) \cap T = \emptyset$, and ${\OB}_{{\cal O}, T} \not \models \neg a$, then $\boldK a \in P$.
The case that $\head(r)$ evaluates to $\bf t$ (under $(M,N)$)  is automatic. 
If 
$\head(r)$ evaluates to $\bf u$, 
i.e., $\boldK a \in P$ and $\boldK a \not \in T$, then 
$\body(r)$ evaluates to $\bf u$ or $\bf f$, since if $\body(r)$ evaluates to $\bf t$, 
$\boldK a \in \lfp(\Phi_{\K}(\cdot, P)_1) (= T)$, resulting in a contradiction. 
If 
$\head(r)$ evaluates to $\bf f$, then ${\OB}_{{\cal O}, T} \models \neg a$, in which case
$\body(r)$ must evaluate to $\bf f$ as well, as otherwise 
$\boldK a \in \theta(T)$ ($= \lfp(\Phi_{\K} (\cdot, T)_1)$) and thus ${\OB}_{{\cal O}, \theta(T)}$ is unsatisfiable, leading to a contradiction.
As this proof applies to all rules in ${\cal P}$,
we have $(M,N) \models \pi(\P)$, and with $(M,N) \models \boldK \pi(\cal O)$, 
$(M,N) \models \pi(\K)$.

Next, for the sake of contradiction assume $(M,N)$ is not a three-valued MKNF model of $\K$. Then there exists a pair $(M',N')$ with $M \subseteq M'$ and $N \subseteq N'$, where at least one of the inclusions is proper and $M' = N'$ if $M = N$, such that 
\begin{eqnarray}
(I,\langle M',N' \rangle,\langle M,N \rangle)(\pi(\K))=\bf t
\label{proof0}
\end{eqnarray}
 for some $I \in M'$. Let  $(T',P')$ be induced by $(M',N')$, i.e., 
$$( M' , N' ) = (\{ I\, |\, I \models {\OB}_{{\cal O} , T'}\},\{I\,|\,I \models {\OB}_{{\cal O},P'}\})$$
Clearly, $T'\subseteq T$ and $P'\subseteq P$, where at least one of the inclusions is proper and $T'=P'$  if $T=P$.
We show that $(I,\langle M',N' \rangle,\langle M,N \rangle) (\pi(\K))\not =\bf t$ (independent of $I$), which leads to a contradiction. 

Consider the case where $T' \subset T$. 
Let the sequence of intermediate sets of 
$\bfK$-atoms in the construction of $\lfp(\Phi_{\K}(\cdot, P)_1)$ be $S_0, \dots, S_n (= T)$.
Assume step $i~(0\leq i \leq n-1)$ is the {\em first} iteration in which 
at least one $\bfK$-atom in $ T \setminus T'$, say $\boldK a$, is derived. By the definition of $\Phi_{\K}(T, P)_1$, the derivation of $\boldK a$ is either by 
${\OB}_{\cO, S_i} \models a$, or by a rule $r \in \P$ such that
$\head(r) = \boldK a$, $\body^+(r) \subseteq S_i$, and $\bfK (\body^-(r)) \cap P = \emptyset$. For the latter case, by the assumption 
that $i$ is the first iteration to derive any $\bfK$-atoms in $T \setminus T'$, $S_i \subseteq T'$. 
It  follows that $\body(r)$ evaluates to $\bf t$ under
$ (I,\langle M',N' \rangle,\langle M,N \rangle)$ (independent of $I$), but its head $\head(r)$ evaluates to either {\bf f} or {\bf u}; thus rule $r$ is not satisfied, resulting in a contradiction to equation (\ref{proof0}). 
If no $\bfK$-atom in $ T \setminus T'$ is ever derived by a rule in iteration $i$, then it must be ${\OB}_{\cO, S_i} \models a$, and along with $S_i \subseteq T'$ and
${\OB}_{\cO, T'} \not \models a$, we derive a contradiction. 

For the case where $P' \subset P$, the proof is similar. Consider the sequence of intermediate sets of $\bfK$-atoms  $Q_0, \dots, Q_n (= P)$ in the iterative construction of $\lfp(\Phi_{\K}(T, \cdot)_2)$. Let $j~(0\leq j \leq n-1)$ be the first iteration in which
at least one $\bfK$-atom in  $P \setminus P'$ is derived (thus $Q_j \subseteq P'$).
Let  $\boldK a$ be such a $\bfK$-atom. Assume it is derived by a rule
 $r \in \P$ with 
$\head(r) = \boldK a$, such that ${\OB}_{{\cO}, T} \not \models \neg a$, $\body^+(r) \subseteq Q_j$, and $\bfK (\body^-(r)) \cap T = \emptyset$. Then, the body of rule $r$ evaluates to {\bf t} or {\bf u} in $(I,\langle M',N' \rangle,\langle M,N \rangle)$. Since $\boldK a \not \in P'$, $\head(r) = \boldK a$ evaluates to {\bf f}, and thus  $(I,\langle M',N' \rangle,\langle M,N \rangle)(\pi(r)) \not = \bf t$. If in iteration $j$ no fresh $\bfK$-atoms are derived by rules, then we must have ${\OB}_{{\cO}, Q_j} \models a$, and along with $Q_j \subseteq P'$ and ${\OB}_{{\cO}, P'} \not \models a$, we reach a contradiction. 
Note that the above proof is naturally applicable when $T' = P'$ because $T = P$, in which case evaluation reduces to two-valued.
As both cases lead to a contradiction, $(M,N)$ is therefore a three-valued MKNF model of $\K$.

($\Rightarrow$)  Let $( M , N ) = (\{ I\, |\, I \models {\OB}_{{\cal O} , T}\},\{I\,|\,I \models {\OB}_{{\cal O},P}\})$
be a three-valued MKNF model of $\K$.
Recall again that 
given an MKNF interpretation pair $(M',N')$, there exists a partition $(X,Y)$ induced by $(M',N')$, in that $( M' , N' ) = (\{ I\, |\, I \models {\OB}_{{\cal O} , X}\},\{I\,|\,I \models {\OB}_{{\cal O},Y}\})$,
 such that for any $\boldK a \in \KA(\K)$,
${\boldK} a \in X$ iff ${\boldK} a$ evaluates to $\bf t$ (under $(M',N')$),
${\boldK} a \not \in Y$ iff ${\boldK} a$  evaluates to $\bf f$, and otherwise ${\boldK} a$ evaluates to $\bf u$.
When $(M,N)$ is a three-valued MKNF model of $\K$, the partition induced by $(M,N)$ is just $(T,P)$ such that  $( M , N ) = (\{ I\, |\, I \models {\OB}_{{\cal O} , T}\},\{I\,|\,I \models {\OB}_{{\cal O},P}\})$. Since $(M,N)$ is a three-valued MKNF model, 
$(T,P)$ is consistent.

We now show that $(T,P)$ is a stable fixpoint of $\Phi_{\K}$. 
First, we show that $(T,P)$ is a fixpoint of $\Phi_{\K}$. By definition, 
$T \subseteq \Phi_{\K}(T,P)_1$. Assume $T \subset \Phi_{\K}(T,P)_1$ and let $\boldK a \not \in T$ and $\boldK a \in \Phi_{\K}(T,P)_1$.
Then there exists a rule $r$ with $\head(r)=\boldK a$ such that $\boldK a$ can be derived due to satisfied body of rule $r$.  
It then follows $(M,N) \not \models \pi(r)$, contradicting to the three-valued MKNF model condition; thus $T=\Phi_{\K}(T,P)_1$. Similarly, we can show  $P=\Phi_{\K}(T,P)_2$. 
If $(T,P)$ is not a stable fixpoint, then either  $T \not = \lfp(\Phi_{\K}(\cdot,P)_1)$ or 
$P \not = \lfp(\Phi_{\K}(T,\cdot)_2)$. For the former case, 
since $T$ is a fixpoint of 
$\Phi_{\K} (\cdot, P)_1$, there exists $T' \subset T$ such that $T'= \lfp(\Phi_{\K}(\cdot,P)_1)$. Consider the partition $(T',P)$, for which we can construct an MKNF interpretation pair $(M',N)$, where $M'=\{ I\, |\, I \models {\OB}_{{\cal O} , T'}\}$ and $M \subset M'$. It can be checked that $(I,\langle M',N \rangle,\langle M,N \rangle)(\pi(\K))=\bf t$, for any $I \in M'$. If $M = N$, one can verify that 
$(I,\langle M',M' \rangle,\langle M,M \rangle)(\pi(\K))=\bf t$ for any $I \in M'$.
Thus, $(M,N)$ is not a 
three-valued MKNF model of $\K$, a contradiction, and thus $T  = \lfp(\Phi_{\K}(\cdot,P)_1)$.
 Similarly, we can show
$P= \lfp(\Phi_{\K}(T, \cdot)_2)$. Therefore, $(T,P)$ is a stable fixpoint of $\Phi_{\K}$, and a consistent one. 

Finally, since $(M,N)$ is a three-valued MKNF model of $\K$, in the construction of $\lfp(\Phi_{\K} (T, \cdot)_2)$
the extra condition ${\OB}_{{\cal O}, T} \not \models \neg a$ in the definition of $\Phi_{\K} (T, \cdot)_2$ always holds whenever the body of the relevant rule evaluates to $\bf t$. It follows $P = \lfp(\Phi_{\K} ( \cdot, T)_1) (= \theta(T))$, and since ${\OB}_{{\cal O}, P}$ is satisfiable,
${\OB}_{{\cal O}, \theta(T)}$ is satisfiable as well.
\end{proof}

{\begin{example}

Consider a hybrid MKNF knowledge base ${\K}=(\{\neg a\},{\P})$, where 
$\P$ consists of
$$
\begin{array}{ll} 
{\boldK} a \leftarrow {\bfK } b.~~~~~~
{\boldK} b \leftarrow  {\boldnot} b.~~~~~~~~~~
\end{array}
$$
The least stable fixpoint of $\Phi_{\K}$ is $(T,P)=(\emptyset, \{\boldK b\})$, which is  consistent but does not  correspond to a three-valued MKNF model since ${\OB}_{\{\neg a\}, \lfp(\Phi_{\K}(\cdot,T)_1)}$ is unsatisfiable.  

\end{example}
}

\comment{
\begin{example} 
We illustrate that inconsistent stable fixpoints are now possible. 
For example,  with ${\K}_1 = (\{\neg a\}, \{a \leftarrow \boldnot b.\})$,  $(\{\boldK a, \boldK b\}, \emptyset)$ is 
a stable fixpoint of $\Phi_{{\cal K}_1}$.

As another example, let  ${\K}_2 = (\{d\}, {\cal P})$ where ${\P} = \{{\boldK} a \leftarrow {\boldK d}, {\boldnot} b.,~\boldK b \leftarrow \boldnot a.\}$. There are four stable fixpoints of  $\Phi_{{\cal K}_2}$, among which $ (\{\boldK d\}, \{\boldK a , \boldK b, \boldK d\})$ is the least fixpoint and $(\{\boldK d,\boldK a, \boldK b \}, \{\boldK d\})$ is an inconsistent one. Note that an inconsistent stable fixpoint may contain consistent information \-- this case in $\boldK d$. 
\end{example}
}

\section{A Richer Approximator for the Well-Founded Semantics}
\label{richer0}

A question arises whether richer approximators for MKNF knowledge bases exist. For any two approximators $A$ and $B$ on $L^2$, $A$ is {\em richer than} $B$ (or {\em more precise than} $B$, in the terminology of \cite{DeneckerMT04}), denoted $B \leq_p A$, if for all $(x,y) \in L^2$, $B(x,y) \leq_p A(x,y)$.

There is a practical motivation for the question. 
Let $(x,y)$ and $(x',y')$ be the least stable fixpoints of $B$ and $A$ respectively. That $A$ is richer than $B$ means $(x,y) \leq_p (x',y')$. If $A$ is strictly richer than $B$, and 
if $(x',y')$ indeed corresponds to the well-founded MKNF model, then $(x,y)$ cannot possibly correspond to the well-founded MKNF model. In this case,   while $(x',y')$
can be computed iteratively for $A$,  
it 
cannot be computed iteratively for $B$. Then, more complex reasoning method must be applied to compute the 
well-founded MKNF model.
We now define such a richer approximator. 
\begin{definition}
\label{Psi}
Let $\K = (\cal O, \P)$ be a hybrid MKNF knowledge base.
Define the operator $\Psi_{\K}$ on $(2^{\KA(\K)})^2$ as follows: 
$\Psi_{\K}(T,P) = (\Psi_{\K} (T,P)_1 ,  \Psi_{\K} (T,P)_2)$,
 where
$$
\begin{array}{ll}
\Psi_{\K} (T,P)_1 = \Phi_{\K} (T,P)_1 \\
\Psi_{\K} (T,P)_2 = 
\{\boldK a\in {\KA}({\K}) \mid {\OB}_{{\cal O}, P}\models a\} \,\,  \cup \,
\\
~~\{ \head(r) \mid  r \in {\P}:  \, \head(r) = \bfK a , \,
{\OB}_{{\cal O}, T} \not \models \neg a,  \, \body^+(r) \subseteq P, \, \boldK(\body^-(r)) \cap T = \emptyset, \\
~\not \exists r' \in {\cal P}: \boldK b \leftarrow \body(r') \mbox{ is positive, where $\boldK a \in \body(r')$, s.t. } {\OB}_{{\cal O}, T}\models \neg b, \body(r') \setminus \{\boldK a\} \subseteq T
\}   
\end{array}
$$
\end{definition}

Operator $\Psi_{\K}$ differs from $\Phi_{\K}$ in the second projection operator, with
an extra condition for deriving $ \boldK a$ (the last line in the definition above), which says that if for some positive rule $r'$ with $\boldK b$ as the head and $\boldK a$ in the body, the objective atom $b$ is already false and the rule's body excluding $\boldK a$ is already true, then, since the rule must be satisfied, $\boldK a$ must be false and thus should not be derived as possibly true. Notice that this is like embedding the {\em unit propagation} rule in automated theorem proving into an approximator. 

\begin{example}
\label{7}
Let 
${\K} = (\{\neg b\},  \P)$, where $\P$ is
$$
\begin{array}{ll}
\boldK b \leftarrow \boldK a, \boldK e. ~~~~~ \boldK e \leftarrow \boldnot p. ~~~~~
\boldK a \leftarrow \boldnot c.  ~~~~~
\boldK c \leftarrow \boldnot a.
\end{array}
$$
The least stable fixpoint of $\Phi_{\K}$ is $(T,P) = (\{\boldK e\}, \{\boldK e, \boldK a, \boldK c\})$, while the least stable fixpoint of  operator $\Psi_{\K}$ is $(T',P') = (\{\boldK e, \boldK c\}, \{\boldK e, \boldK c\})$, which 
corresponds to the well-founded MKNF model of $\K$. Note that  $(T',P')$ is also a stable fixpoint of $\Phi_{\K}$; but  because it is not the least, it cannot be computed by the standard iterative process.
\end{example}

\begin{proposition}
Operator $\Psi_{\K}$ is an approximator for ${\cal T}_{\K}$.
\end{proposition}

\begin{proof}
We can verify that $\Psi_{\K}$ is $\subseteq$-monotone on $(2^{\KA(\K)})^2$. Let $(T_1,P_1) \subseteq_p (T_2,P_2)$. That $\Psi_{\K}(T_1,P_1)_1 \subseteq \Psi_{\K} (T_2,P_2)_1$ is immediate by definition. To show $\Psi_{\K}(T_2,P_2)_2 \subseteq \Psi_{\K} (T_1,P_1)_2$,  we check all conditions in the definition of $\Psi_{\K}(T,P)_2$; in particular, let us consider the following conditions in the definition of $\Psi_{\K}(T,P)_2$: 
\begin{eqnarray}
{\OB}_{{\cal O}, T} \not \models \neg a~~~~~~~~~~~~~~~~~~~~~~~~~~~~~~~~~~~~~~~~~~~~~~~~~~~~~~~~~~~~~~~~~~~~ \label{2}\\
\not \exists r' \in {\cal P}: \boldK b \leftarrow \body(r') \mbox{ is positive, where } \boldK a \in \body(r'), \nonumber\\
~~~~~~~~~~~~~~~~~\mbox{ s.t. } {\OB}_{{\cal O}, T}\models \neg b \mbox{ and } \body(r') \setminus \{\boldK a\} \subseteq T  ~~ \label{3}
\label{psi}
\end{eqnarray}
which may block a $\bfK$-atom $\bfK \,a$ in the definition (cf. the second subset in the definition) to be included. There are three conditions in these expressions (one in (\ref{2}) and two in (\ref{3}))  
that are determined by  the following relationships under $(T_1,P_1) \subseteq_p (T_2,P_2)$:
${\OB}_{{\cal O}, T_2} \models{\OB}_{{\cal O}, T_1}$ and  $T_2 \models T_1$. 
It is then easy to check that  $\Psi_{\K}(T_2,P_2)_2 \subseteq \Psi_{\K} (T_1,P_1)_2$, and  we therefore have
$\Psi_{\K}(T_1,P_1) \subseteq_p  \Psi_{\K}(T_2,P_2)$. 
Furthermore, $\Psi_{\K}$ approximates ${\cal T}_{\K}$, since by definition $\Psi_{\K}(I, I)_1 \supseteq \Psi_{\K}(I, I)_2$, and it follows that 
whenever $\Psi_{\K}(I, I)$ is consistent,
$\Psi_{\K}(I, I) = ({\cal T}_{\K} (I), {\cal T}_{\K} (I))$.
\end{proof}


We show that $\Psi_{\K}$ is more precise than $\Phi_{\K}$.

\begin{proposition}
\label{Phi_Psi}
Given any hybrid MKNF knowledge base $\K$, 
$\Phi_{\K} \subseteq_p \Psi_{\K}$.
\end{proposition}

\begin{proof}
This is due to the extra condition in the definition of $\Psi_{\K}(T,\cdot)_2$, which is not present in the definition of $\Phi_{\K}(T,\cdot)_2$. A stronger condition produces a subset for the second component of the resulting pair.  Thus, if $(T_1,P_1) \subseteq_p (T_2,P_2)$, from $\Psi_{\K}(T,x)_1 =  \Phi_{\K}(T,x)_1$ for
any $x \subseteq \KA(\K)$, it follows
that 
 $\Phi_{\K}(T_1,P_1) \subseteq_p \Psi_{\K} (T_2,P_2)$.
\end{proof}

Operator $\Psi_{\K}$ is strong as well.
\begin{proposition}
\label{strong1}
Let $\K = (\cal O, \P)$ be a hybrid MKNF knowledge base and $(T,P)$ be a consistent stable fixpoint of $\Psi_{\K}$ such that 
${\OB}_{{\cal O}, T}$ is satisfiable. Then  $\Psi_{\K}$ is a strong approximator for $(T,P)$.
\end{proposition}


\begin{proof}
Let $(T,P)$ be a consistent stable fixpoint of $\Psi_{\K}$ such that 
${\OB}_{{\cal O}, T}$ is satisfiable.
We need to prove
$(C_1(P), C_1(T))= (D_1(P), D_2(T))$. The proof is identical to that of  Proposition \ref{strong0} except for the proof of 
$D_1(P)  \subseteq C_2(T)$, which needs to be updated according to the definition of operator $\Psi_{\K}(T,\cdot)_2$. Recall that the goal is to show $T \subseteq C_2(T)$.
The base case  is again  immediate. Assume 
$D_1^{\uparrow k}(P) \subseteq C_2^{\uparrow k} (T)$ for any (fixed) $k\geq 0$, and we show it for $k+1$. By definition, a $\bfK$-atom $\bfK \, a$ is added to $D_1^{k+1}(P)$ because (i)  ${\OB}_{{\cal O}, D_1^{\uparrow k}(P) }\models a$, or (ii) there is a rule
$r \in {\cal P}$ with $\head(r) = \bfK \, a$ such that 
$\body^+(r) \subseteq D_1^{\uparrow k}(P)$ and $ \boldK(\body^-(r)) \cap P = \emptyset$. 
Again, if case (i) applies, by induction hypothesis (I.H.), ${\OB}_{{\cal O}, C_2^{\uparrow k}(T) }\models a$ and thus $\bfK \, a \in C_2^{\uparrow k+1}(T)$. Otherwise, $\bfK \, a$ is derived by a rule $r$, as in case (ii). Since $D_1^{\uparrow k}(P) \subseteq T$ and  $ \boldK(\body^-(r)) \cap P = \emptyset$,
 case (ii) implies ${\OB}_{{\cal O}, T} \models a$.  If ${\OB}_{{\cal O}, T} \models \neg a$, then ${\OB}_{{\cal O}, T}$ is unsatisfiable, which is a contradiction.
Thus,
we must have ${\OB}_{{\cal O}, T} \not \models \neg a$. Now consider the condition
$$
\not \exists r' \in {\cal P}: \boldK b \leftarrow \body(r') \mbox{ is positive, where } \boldK a \in \body(r'),  \mbox{ s.t. } {\OB}_{{\cal O}, T}\models \neg b, \body(r') \setminus \{\boldK a\} \subseteq T$$
in the definition of $\Psi_{\K}(T,\cdot)_2$. By definition, $\boldK a$ is not derived by rule $r$ in applying $\Psi_{\K}(T,\cdot)_2$
  if such a rule $r'$ exists.
Since  ${\OB}_{{\cal O}, T} \models a$,  from 
$\body(r') \setminus \{\boldK a\} \subseteq T$ we derive $\body(r')  \subseteq T$, and therefore $\boldK b$ can be derived by rule $r'$ resulting in $\boldK b \in T$,  but at the same time we have ${\OB}_{{\cal O}, T}\models \neg b$, and thus ${\OB}_{{\cal O}, T}$ is unsatisfiable, a contradiction. Hence, such a rule $r'$ does not exist. Therefore,  the same rule applied in case (ii) for $D_1^{\uparrow k+1}$
above applies in the construction of $C_2^{\uparrow k+1}$, since the condition $(\body^+(r) \subseteq D_1^{\uparrow k}(P)$ and $ \boldK(\body^-(r)) \cap P = \emptyset)$ becomes $(\body^+(r) \subseteq C_2^{\uparrow k}(T)$ and $ \boldK(\body^-(r)) \cap T = \emptyset)$), which holds 
by I.H. and $T \subseteq P$. Thus, $D_1(P)  ~(=T) \subseteq C_2(T)$.
\end{proof}

\comment{
\begin{proposition}
\label{Psi_stable}
Let $(T,P)$ be a consistent stable fixpoint of $\Psi_{\K}$ such that ${\OB}_{{\cal O},T}$ is satisfiable.
Then $(T,P)$ is also a consistent stable fixpoint of $\Phi_{\K}$.
\end{proposition}

\begin{proof}
That $(T,P)$ is a stable fixpoint of $\Psi_{\K}$ implies $T=\lfp(\Psi_{\K}(\cdot,P)_1)$ and $P=\lfp(\Psi_{\K}(T,\cdot)_2)$. Since $\Psi_{\K}(\cdot,P)_1=\Phi_{\K}(\cdot,P)_1$, it is naturally 
$T=\Phi_{\K}(T,P)_1=\Psi_{\K}(T,P)_1$. As to $P$, let $P'=\lfp(\Phi_{\K}(T,\cdot)_2)$,  by Proposition \ref{Phi_Psi}, it can be easily checked that $P \subseteq P'$. 
If $P \subsetneq P'$, then $\exists {\boldK}a \in P'$ and ${\boldK}a \not \in P$ and ${\boldK}a \not \in T$.
Let $H=\{{\boldK a} \mid r \in {\cal P}, \head(r)={\boldK}a, \boldK(\body^-(r)) {\color{red} \subseteq} P, \boldK(\body^-(r)) \cap T=\emptyset\}$; and for each such $r$ in it,
$$\exists r' \in {\cal P}: \boldK b \leftarrow \body(r') \mbox{ is positive such that } {\OB}_{{\cal O}, T}\models \neg b, \body(r') \setminus \{\boldK a\} \subseteq T\}$$
Then, we have $P=P' \setminus H$.
Since $P \subset P'$, $\Phi_{\K}(T,P')_1 \subseteq \Phi_{\K}(T,P)_1$.
If $\exists r^*$ with $\head(r^*) = \boldK c$,
where for some $\boldK a \in H$, $a \in \body^-(r^*)$, {\color{red} (If not this, $H$ may have no impact here)}
s.t. $\body^+(r^*) \subseteq T$, $\boldK(\body^-(r^*)) \cap H \neq \emptyset$, {\color{red} (what if $\not \exists$ such an $r^*$?)}
then we have ${\boldK}c \in \Phi_{\K}(T,P)_1$, ${\boldK}c \not \in \Phi_{\K}(T,P')_1$, contradicted to $T=\Phi_{\K}(T,P)_1=\Phi_{\K}(T,P')_1$, {\color{red} (This is close, but why $T = \Phi_{\K}(T,P')_1$?)}
therefore,
to let $T=\Phi_{\K}(\cdot,P)_1=\Phi_{\K}(\cdot,P')_1$, {\color{red} (Do you assume it, or it's a fact that can be proved?)}  it is needed that 
$\forall r \in {\cal P}$, it should have $\boldK(\body^-(r)) \cap H =\emptyset$, but just as example \ref{7} shows,  that is not the case.
Then $P=P'$, $(T,P)$ is also the stable fixpoint of $\Phi_{\K}$.
\end{proof}
}

\comment{
\begin{proof} {\bf  (Old proof)}
That $(T,P)$ is a stable fixpoint of $\Psi_{\K}$ implies $T=\lfp(\Psi_{\K}(\cdot,P)_1)$ and $P=\lfp(\Psi_{\K}(T,\cdot)_2)$. Since $\Psi_{\K}(\cdot,P)_1=\Phi_{\K}(\cdot,P)_1$, it is naturally 
$T=\Phi_{\K}(\cdot,P)_1$. As to $P$, let $P'=\lfp(\Phi_{\K}(T,\cdot)_2)$,  by Proposition \ref{Phi_Psi}, it can be easily checked that $P \subseteq P'$. 
If $P \subsetneq P'$, then $\exists {\boldK}a \in P'$ and ${\boldK}a \not \in P$ and ${\boldK}a \not \in T$.
Let $H=\{{\boldK a} \mid \exists r \in {\cal P}, \head(r)={\boldK}a, \boldK(\body^-(r)) \subset P, \boldK(\body^-(r)) \cap T=\emptyset; and 
\exists r' \in {\cal P}: \boldK b \leftarrow \body(r') \mbox{ is positive such that } {\OB}_{{\cal O}, T}\models \neg b, \body(r') \setminus \{\boldK a\} \subseteq T\}$,
then we have $P=P' \setminus H$.
As to $\Phi_{\K}(T,P)_1$ and $\Phi_{\K}(T,P')_1$, since $\Phi_{\K}(T,P')_1 \subseteq \Phi_{\K}(T,P)_1$ by $P \subset P'$, 
if $\exists r^*$ with the form ${\boldK}c \leftarrow \boldK(\body^+(r^*)), \boldK(\body^-(r^*))$, s.t., $\boldK(\body^+(r^*)) \subseteq T$, $\boldK(\body^-(r^*)) \cap H \neq \emptyset$, 
then we have ${\boldK}c \in \Phi_{\K}(T,P)_1$, ${\boldK}c \not \in \Phi_{\K}(T,P')_1$, contradicted to $T=\Phi_{\K}(\cdot,P)_1=\Phi_{\K}(\cdot,P')_1$, therefore,
to let $T=\Phi_{\K}(\cdot,P)_1=\Phi_{\K}(\cdot,P')_1$, it is needed that 
$\forall r \in {\cal P}$, it should have $\boldK(\body^-(r)) \cap H =\emptyset$, but just as example \ref{7} shows,  that is not the case.
Then $P=P'$, $(T,P)$ is also the stable fixpoint of $\Phi_{\K}$.
\end{proof}
}

\comment{
{\color{blue}

\begin{proposition}
\label{strong1}
Let $\K = (\cal O, \P)$ be a hybrid MKNF knowledge base and $(T,P)$ be a consistent stable fixpoint of $\Psi_{\K}$ such that 
${\OB}_{{\cal O}, T}$ is satisfiable. Then  $\Psi_{\K}$ is a strong approximator for $(T,P)$.
\end{proposition}

\begin{proof}
For $\Psi_{\K}$, to show it is strong, we show that given consistent stable fixpoint $(T,P)$ of $\Psi_{\K}$, $C_2(T)$, $D_2(T)$ denote the least fixpoints of $\Psi_{\K}(T,\cdot)_2$ defined on $[\bot,\top]$ and $[T,\top]$ respectively,
$C_2(T)$ and $D_2(T)$ coincide and $C_2(T)=D_2(T)=P$. 

Assume that there exists a stable fixpoint $(T,P')$ of $\Psi_{\K}$, s.t., $C_2(T)=P'$ and $P' \lneq T$, then by Prop.\ref{Psi_stable}, $(T,P')$ also a stable fixpoint of $\Phi_{\K}$. On the other side, by Prop.\ref{strong0},
$\Phi_{\K}$ is strong, let $C'_2(T)$, $D'_2(T)$ denote the least fixpoints of $\Phi_{\K}(T,\cdot)_2$ defined on $[\bot,\top]$ and $[T,\top]$ respectively, $C'_2(T)=D'_2(T)=P$, then $(T,P')$ is not a stable fixpoint of $\Phi_{\K}$,
contradiction. Therefore,  $C_2(T)=D_2(T)=P$.
\end{proof}
}}

Finally, like Theorem \ref{key-result}, the stable fixpoints of the operator $\Psi_{\K}$ can be related to three-valued MKNF models of $\K$ as well.

\begin{theorem}
\label{key-result0}
Let $\K = (\cal O,\P)$ be a hybrid MKNF knowledge base and $(T,P)$ be a partition. Let further $( M , N ) = (\{ I\, |\, I \models {\OB}_{{\cal O} , T}\},\{I\,|\,I \models {\OB}_{{\cal O},P}\})$.
Then, 
$(M,N)$ is a three-valued MKNF model of $\K$ iff $(T,P)$ is a consistent stable fixpoint of $\Psi_{\K}$  and 
${\OB}_{{\cal O}, \lfp(\Psi_{\K}(\cdot,T)_1)}$ is satisfiable. 
\end{theorem}

\begin{proof}
The proof here follows the structure of the proof of Theorem \ref{key-result}, but with critical differences in dealing with  approximator $\Psi_{\K}$. If a part of proof of Theorem \ref{key-result} can be applied, we will make a reference to it, otherwise we will give a detailed proof even if parts of it repeat the same from the proof for Theorem \ref{key-result}.

($\Leftarrow$)
Assume that $(T,P)$ is a consistent stable fixpoint of $\Psi_{\K}$ and ${\OB}_{{\cal O}, \lfp(\Psi_{\K}(\cdot,T)_1)}$
is satisfiable. Let $\theta(x)$ denote $\lfp(\Psi_{\K}(\cdot,x)_1)$, given $x \subseteq  \KA({\K})$.
Let $P^*=\lfp(\Phi_{\K}(T,\cdot)_2)$, then by the definition of operator $\Phi_{\K}$, 
 $P^* \subseteq \theta(T)$,
and by Proposition \ref{Phi_Psi}, $\Psi_{\K}(T,P^*)_2 \subseteq \Phi_{\K}(T,P^*)_2$, i.e., $\Psi_{\K}(T,P^*)_2 \subseteq P^*$ and thus $P^*$ is a pre-fixpoint of $\Psi_{\K}(T,\cdot)_2$; 
therefore $P \subseteq P^*$ and then $P \subseteq \theta(T)$. Since ${\OB}_{{\cal O}, \theta(T)}$ is satisfiable, it follows that 
${\OB}_{{\cal O}, P}$ is satisfiable, and because $T \subseteq P$, ${\OB}_{{\cal O}, T}$ is also satisfiable.
It follows that the pair $$( M , N ) = (\{ I\, |\, I \models {\OB}_{{\cal O} , T}\},\{I\,|\,I \models {\OB}_{{\cal O},P}\})$$
is an MKNF interpretation pair because 
$\emptyset \subset N \subseteq M$. Recall the property by \cite{KnorrAH11}: given the above interpretation pair, for any $\boldK \xi \in \KA(\K)$, ${\boldK} \xi \in T$ iff ${\boldK} \xi$ evaluates to $\bf t$ (under $(M,N)$),
${\boldK} \xi \not \in P$ iff ${\boldK} \xi$  evaluates to $\bf f$,  and otherwise ${\boldK} \xi$ evaluates to $\bf u$.

We show that $(M,N)$ is a three-valued MKNF model of $\K$. 
First we show that $(M,N)$ satisfies $\pi({\K})$. The proof that $(M,N) \models \boldK \pi({\cal O})$ is straightforward. For rules in $\P$, recall the following definition of $\Psi_{\K}(x,y)_2$:
$$
\begin{array}{ll}
\Psi_{\K} (x,y)_2 = 
\{\boldK a\in {\KA}({\K}) \mid {\OB}_{{\cal O}, y}\models a\} \,\,  \cup \,
\\
~~~\{ \head(r) \mid  r \in {\P}:  \, \head(r) = \bfK a , \,
{\OB}_{{\cal O}, x} \not \models \neg a,  \, \body^+(r) \subseteq y, \, \boldK(\body^-(r)) \cap x = \emptyset, \\
~~\not \exists r' \in {\cal P}: \boldK b \leftarrow \body(r') \mbox{ is positive, where } \boldK a \in \body(r'), \mbox{ s.t. } {\OB}_{{\cal O}, x}\models \neg b, \body(r') \setminus \{\boldK a\} \subseteq x
\}   
\end{array}
$$
where the only difference from  $\Phi_{\K}$ 
is the extra condition in the last line above. It can be checked that this extra condition does not effect the proof used for 
Theorem \ref{key-result}. Namely, for any rule $r \in \P$, it is satisfied if $\head(r)$ evaluates to $\bf t$;
if
$\head(r)$ evaluates to $\bf u$, which means $\boldK a \in P$ and $\boldK a \not \in T$, then 
$\body(r)$ evaluates to $\bf u$ or $\bf f$, since if $\body(r)$ evaluates to $\bf t$, 
$\boldK a \in \lfp(\Psi_{\K}(\cdot, P)_1) (= T)$, resulting in a contradiction; 
if 
$\head(r)$ evaluates to $\bf f$, then ${\OB}_{{\cal O}, T} \models \neg a$, in which case
$\body(r)$ evaluates to $\bf f$ as well, as otherwise 
$\boldK a \in \theta(T)$ ($= \lfp(\Psi_{\K} (\cdot, T)_1)$) and thus ${\OB}_{{\cal O}, \theta(T)}$ is unsatisfiable, again a contradiction. Hence, we conclude that $(M,N) \models \pi({\P})$, and therefore $(M,N) \models \pi({\K})$. 

Now, for the sake of contradiction, assume  $(M,N)$ is not a three-valued MKNF model of $\K$. Then there exists an MKNF interpretation  pair $(M',N')$ with $M \subseteq M'$ and $N \subseteq N'$, where 
at least one of the inclusions is proper and $M' = N'$ if $M = N$, such that 
\begin{eqnarray}
(I,\langle M',N' \rangle,\langle M,N \rangle)(\pi(\K))=\bf t
\label{proof1}
\end{eqnarray}
 for some $I \in M'$. Let  $(T',P')$ be induced by $(M',N')$, i.e., 
$$( M' , N' ) = (\{ I\, |\, I \models {\OB}_{{\cal O} , T'}\},\{I\,|\,I \models {\OB}_{{\cal O},P'}\})$$
Clearly, $T'\subseteq T$ and $P'\subseteq P$, where at least one of the inclusions is proper and $T'=P'$  if $T=P$.
We show that $(I,\langle M',N' \rangle,\langle M,N \rangle) (\pi(\K))\not =\bf t$ (independent of $I$), which leads to  contradiction. 

Consider the case where $T' \subset T$. As in the proof of Theorem \ref{key-result}, this part of proof relies on the fixpoint construction of $\lfp(\Psi_{\K}(\cdot, P)_1)$.
Since by definition $\Psi_{\K}(T, P)_1 = \Phi_{\K}(T, P)_1$, the construction of $\lfp(\Psi_{\K}(\cdot, P)_1)$ is identical to that of $\lfp(\Phi_{\K}(\cdot, P)_1)$, the same proof of Theorem \ref{key-result} for this part can be applied here, which shows 
$(I,\langle M',N' \rangle,\langle M,N \rangle)(\pi(\K))\not =\bf t$.

For the case of $P' \subset P$, consider the sequence of intermediate sets of $\bfK$-atoms  $Q_0, \dots, Q_n (= P)$ in the iterative construction of $\lfp(\Psi_{\K}(T, \cdot)_2)$. Let $j~(0\leq j \leq n-1)$ be the first iteration in which
at least one $\bfK$-atom, say $\boldK a \in  P \setminus P'$, is derived (thus $Q_j \subseteq P'$).
Assume further it is derived by a rule $r \in P$ with $\head(r)=\boldK a$, such that ${\OB}_{{\cO}, T} \not \models \neg a$,
$\body^+(r) \subseteq Q_j$, $\bfK (\body^-(r)) \cap T=\emptyset$, and  $\not \exists r' \in {\cal P}: \boldK b \leftarrow \body^+(r') \mbox{ which is positive, where } \boldK a \in \body(r'), \mbox{ s.t. } {\OB}_{{\cal O}, T}\models \neg b$ and $ \body^+(r') \setminus \{\boldK a\} \subseteq T$. Then, it can be seen that the body of rule $r$ evaluates to {\bf t} or {\bf u} in $(I,\langle M',N' \rangle,\langle M,N \rangle)$. Since $\boldK a \not \in P'$, $\head(r) = \boldK a$ evaluates to {\bf f}, and thus  $(I,\langle M',N' \rangle,\langle M,N \rangle)(\pi(r)) \not = \bf t$. If in iteration $j$ no fresh $\bfK$-atoms are derived by rules, then we must have ${\OB}_{{\cO}, Q_j} \models a$, and along with $Q_j \subseteq P'$ and ${\OB}_{{\cO}, P'} \not \models a$, we have a contradiction. 
Note that the above proof naturally applies when $T' = P'$ because of $T = P$, in which case evaluation reduces to two-valued.
We therefore conclude that $(M,N)$ is a three-valued MKNF model of $\K$.


($\Rightarrow$) Let 
$( M , N ) = (\{ I\, |\, I \models {\OB}_{{\cal O} , T}\},\{I\,|\,I \models {\OB}_{{\cal O},P}\})$
be a three-valued MKNF model of $\K$.  As mentioned earlier, the following property holds: for any $\boldK \xi \in {\KA}({\K})$, 
${\boldK} \xi \in T$ iff ${\boldK} \xi$ evaluates to $\bf t$ (under $(M,N)$),
${\boldK} \xi \not \in P$ iff ${\boldK} \xi$  evaluates to $\bf f$, and otherwise ${\boldK} \xi$ evaluates to $\bf u$.

 By Theorem  \ref{key-result}, $(T,P)$ is a consistent stable fixpoint of $\Phi_{\K}$  and 
${\OB}_{{\cal O}, \lfp(\Phi_{\K}(\cdot,T)_1)}$ is satisfiable. By definition, $\Psi_{\K}(\cdot, x)_1$ is the same operator as $ \Phi_{\K}(\cdot,x)_1$ for all $x \subseteq \KA({\K})$, and it follows ${\OB}_{{\cal O}, \lfp(\Psi_{\K}(\cdot,T)_1)}$ is also satisfiable and 
$T = \lfp(\Psi_{\K}(\cdot,P)_1)$. Thus,  for $(T,P)$ to be a stable fixpoint of $\Psi_{\K}$, we only need to show 
 $P = \lfp(\Psi_{\K}(T, \cdot)_2)$.  Let $P'  = \lfp(\Psi_{\K}(T, \cdot)_2)$. By definition (due to the extra condition in
the definition of $\Psi_{\K}(T, \cdot)_2$), $P' \subseteq P$. For a contradiction, assume  $P' \subset P$. Let $\boldK a \in P$ and $\boldK a \not \in P'$.
Then,  by the definition of $\Psi_{\K}(T, \cdot)_2$, the reason for $\boldK a \not \in P'$ is that, for
any 
rule $r \in {\cal P}$ with $\head(r) = \boldK a$ such that 
$
{\OB}_{{\cal O}, T} \not \models \neg a$,  
  $\body^+(r) \subseteq P$, and $ \boldK(\body^-(r)) \cap T = \emptyset$, 
there exists a rule $r' \in {\cal P}: \boldK b \leftarrow \body^+(r') \mbox{ which is positive, where } \boldK a \in \body(r') \mbox{ s.t. } {\OB}_{{\cal O}, T}\models \neg b$ and $ \body^+(r') \setminus \{\boldK a\} \subseteq T$. 
If $\boldK a \in T$, then $ {\OB}_{{\cal O}, T}\models  b$ and $ {\OB}_{{\cal O}, T}$ is thus unsatisfiable, contradicting to the fact that 
$( M , N )$ is an MKNF interpretation pair. If $\boldK a \not \in T$, since $\boldK a \in P$, the truth value of $\boldK a$ is {undefined} in $(M,N)$, and thus $r'$ is not satisfied by $( M , N )$; again a contradiction. Thus, $P' = P$, and therefore $(T,P)$ is a stable fixpoint of $\Psi_{\K}$.
\end{proof}

\comment{
\begin{proof}
($\Leftarrow$) Let $(T,P)$ be a consistent stable fixpoint of $\Psi_{\K}$  and 
${\OB}_{{\cal O}, \lfp(\Psi_{\K}(\cdot,T)_1)}$ is satisfiable. 
By Proposition \ref{Psi_stable}, $(T,P)$ is a stable fixpoint of $\Phi_{\K}$. Since by definition $\Phi_{\K} (x,y)_1 = \Psi_{\K} (x,y)_1$ for any pair $(x,y)$,  hence  ${\OB}_{{\cal O}, \lfp(\Phi_{\K}(\cdot,T)_1)}$ is also satisfiable. By Theorem  \ref{key-result}, $( M , N ) = (\{ I\, |\, I \models {\OB}_{{\cal O} , T}\},\{I\,|\,I \models {\OB}_{{\cal O},P}\})$
is a three-valued MKNF model of $\K$. 

($\Rightarrow$) Let 
$( M , N ) = (\{ I\, |\, I \models {\OB}_{{\cal O} , T}\},\{I\,|\,I \models {\OB}_{{\cal O},P}\})$
be a three-valued MKNF model of $\K$. By Theorem  \ref{key-result}, $(T,P)$ is a consistent stable fixpoint of $\Phi_{\K}$  and 
${\OB}_{{\cal O}, \lfp(\Phi_{\K}(\cdot,T)_1)}$ is satisfiable, so is ${\OB}_{{\cal O}, \lfp(\Psi_{\K}(\cdot,T)_1)}$. 
We are done if $(T,P)$ is also a stable fixpoint of $\Psi_{\K}$. Assume it is not. Then, there exists $\boldK a \in P$ whose derivation by
operator $\Psi_{\K}(T,\cdot)_2$ is blocked, i.e.,
any 
rule $r \in {\cal P}$ with $\head(r) = \boldK a$ such that 
$
{\OB}_{{\cal O}, T} \not \models \neg a$ and 
  $\body^+(r) \subseteq P, \,\, \boldK(\body^-(r)) \cap T = \emptyset$, 
there exists a rule $r' \in {\cal P}: \boldK b \leftarrow \body(r') \mbox{ is positive such that } {\OB}_{{\cal O}, T}\models \neg b, \body(r') \setminus \{\boldK a\} \subseteq T$. If $\boldK a \in T$, then $ {\OB}_{{\cal O}, T}\models  b$, which results in $ {\OB}_{{\cal O}, T}$ being unsatisfiable contradicting to the fact that 
$( M , N )$ is a three-valued MKNF model of $\K$. If $\boldK a \not \in T$, since $\boldK a \in P$, the truth value of $\boldK a$ is {undefined} in $(M,N)$ and thus rule $r'$ is not satisfied by $( M , N )$, again, a contradiction.
Therefore, $(T,P)$ must be a stable fixpoint of $\Psi_{\K}$ and we already showed that ${\OB}_{{\cal O}, \lfp(\Psi_{\K}(\cdot,T)_1)}$ is satisfiable.
\end{proof}
}

\comment{
What if a well-founded induction is computed and the test on ${\OB}_{{\cal O}, \Gamma_{\K}(T))}$ turns out to be unsatisfiable.  We know that $(M,N)$ is not a well-founded MKNF model.  The following example shows that a well-founded MKNF model may well exist but well-founded inductions are unable to construct it. 

\begin{example}
Let 
${\K} = (\{\neg a\},  \P)$, where $\P$ is
$$
\begin{array}{ll}
\boldK a \leftarrow \boldK b, \boldK c. ~~~~~
\boldK b \leftarrow \boldnot b.
~~~~~
\boldK c \leftarrow \boldnot d.  ~~~~~
\boldK d \leftarrow \boldnot c.
\end{array}
$$
The unique three-valued MKNF model, which is therefore a well-founded MKNF model, corresponds to partition $(\{\boldK d\}, \{\boldK d, \boldK b\})$.
That is, because $\boldK a$ is false, either $\boldK b$ or $\boldK c$ is false. But we don't know which one unless we compute further. 
\end{example}
}

\section{Summary, Related Work and Remarks}
\label{related}

The primary goal of this paper is to show that the alternating fixpoint operator formulated by Knorr et al. \shortcite{KnorrAH11} for hybrid MKNF knowledge bases is in fact an approximator of AFT, which can therefore be applied to characterize the well-founded semantics, two-valued semantics, as well as three-valued semantics for hybrid MKNF knowledge bases, and enables a better understanding of  the relationships between these semantics in terms of a lattice structure. 

Since this alternating fixpoint operator can map a consistent state to an inconsistent one,
the desire to support operators like this 
motivated us to develop a mild generalization of AFT. As a result, all approximators defined on the entire product bilattice are well-defined
without the assumption of symmetry as required in the original AFT. In this paper, we studied the subtle issue whether consistent stable fixpoints can be preserved in the generalized AFT, and showed that for both approximators formulated in this paper for hybrid MKNF knowledge bases, consistent stable fixpoints are indeed carried over.


The alternating fixpoint construction by Knorr et al. aims at an iterative computation of the well-founded model. In \cite{LY17}, this construction is related to a notion called {\em stable partition} which exhibits properties corresponding to three-valued MKNF models. Based on the notion of stable partition, the relations between the alternating fixpoint construction and three-valued MKNF models are established. 
In this work, we do not use stable partition, instead we characterize three-valued MKNF models directly in terms of stable fixpoints of two appropriate approximators. 
 In this way, we are able to show that the two approximators that are defined on the entire product bilattice capture the consistent stable fixpoints that lead to three-valued MKNF models, even though these approximators may have inconsistent stable fixpoints.

The only other work that treats inconsistency in AFT explicitly is \cite{BiYF14}, where in case of inconsistency, instead of computing $(\lfp({A}( \cdot , v)_1), \lfp(A(u, \cdot)_2))$ on the respective domains $[\bot, v]$ and $[u, \top]$, one computes $(\lfp({A}( \cdot , v)_1), A(u,v)_2)$ because $\lfp(A(u, \cdot)_2)$ may no longer be defined on $[u, \top]$.  By computing $A(u,v)_2$ for the second component of the resulting pair, non-minimal elements may be computed as sets of possibly true atoms when inconsistency arises.


The possibility of accommodating inconsistencies in AFT was first raised in \cite{denecker2000approximations}. The precision order when applied to inconsistent pairs can be regarded as an order that measures the ``degree of inconsistency", or ``degree of doubt"  \cite{DeneckerMT04}.
If two inconsistent pairs satisfy $(x,y) \leq_p (x',y')$, 
the latter can be viewed at least as inconsistent as the former.  
 In a more general context, researches have been trying to address questions like ``where is the inconsistency", ``how severer is it", and how to make changes to an inconsistent theory (see, e.g., \cite{BonaHunter17}).  
A deeper understanding of inconsistencies in the context of AFT presents an interesting future direction.

In answer set programming, researchers have studied paraconsistent semantics. A noticeable example is 
the semi-stable semantics proposed by Sakama and Inoue \shortcite{DBLP:journals/logcom/SakamaI95}
for extended disjunctive logic programs, where a program transformation, called epistemic transformation, is introduced which embodies a notion of ``believed to hold". The semantics is then characterized and enhanced by Amendola
et al. \shortcite{DBLP:journals/ai/AmendolaEFLM16} 
using pairs of interpretations in the context of 
the logic of here-and-there \cite{DBLP:conf/iclp/PearceV08}.  For hybrid MKNF knowledge bases, Kaminski et al.  \shortcite{DBLP:conf/ijcai/KaminskiKL15} propose a five-valued and a six-valued semantics for paraconsistent reasoning with different kinds of inconsistencies.  An interpretation in this context is called a $p$-interpretation which evaluates a formula to true, false, or inconsistent. 
Since these semantics are formulated using the semantic structure consisting of pairs of interpretations, it is interesting to see whether appropriate approximators can be formulated to characterize intended models (of course,  for the non-disjunctive case only since current AFT does not support disjunctive rules). In particular, since the alternating fixpoint constructions are already defined for Kaminski et al.'s five-valued and six-valued semantics,  it may be possible to recast such an alternating fixpoint operator by an approximator. If successful, an interesting result would be that the underlying approximator defines not only the well-founded semantics but also five-valued and six-valued stable semantics. 
Furthermore, like \cite{Ji17},
due to the lattice structure of stable fixpoints, it may be possible to develop 
a DPLL-style solver for these knowledge bases based on a computation of unfounded atoms. 


For disjunctive hybrid MKNF knowledge bases, the state-of-the-art reasoning method is still based on guess-and-verify as formulated by \cite{Motik:JACM:2010}. The lack of conflict-directed reasoning methods has prevented the theory from being tested in practice. Before any attempt to build a solver,  one critical issue to study is the notion of unfounded sets for disjunctive hybrid MKNF knowledge bases, which has recently been investigated by \cite{spencerKR21},
and another one is to develop a conflict-driven search engine for computing MKNF models.

\comment{
\begin{appendix}
\section{A strong approximator for Logic Programming}
\label{LP}

We show that a small adjustment to Fitting's immediate consequence operator  \cite{Fitting02} makes it a strong approximator.

We consider a set of ground atoms $At$. A {\em four-valued interpretation} is a pair $(T,P) \in (2^{At})^2$.  A {\em two-valued interpretation} is a four-valued interpretation of the form $(I,I)$, which we often just write as $I$ and call it an {\em interpretation}, and a {\em three-valued (or partial)  interpretation} $(T,P)$ is a four-valued interpretation such that $T \subseteq P$. As usual, an interpretation $(T,P)$ is understood as: $T$ is the set of atoms interpreted $true$, $P$ the set of atoms interpreted {\em possibly true} and thus $At \setminus P$ is the set of {\em false} atoms.  When the sets of true and false atoms are non-overlapping, it is a three-valued interpretation, otherwise it is an {\em inconsistent interpretation};
when a three-valued interpretation is such that the union of true and false atoms is $At$, it is a two-valued interpretation. 

A {\em (normal) logic program} ${\cal P}$ is a set of rules $r$ of the form $a \leftarrow l_1, \dots, l_n$, where $a \in At$, called the {\em head} and denoted $\head(r)$, and the atoms appearing in $l_i$'s are from $At$ and partitioned into two sets, the set of atoms appearing in $r$ positively, denoted $\body^+(r)$, and the set of atoms $b$
such that $b$ appears in $r$ negatively, 
 denoted $\body^-(r)$.  We use $\wholebody(r) = \{l_1, \dots, l_n\}$ to represent the whole rule body.
We  define the familiar immediate consequence operator for two-valued interpretations as follows: given $I \subseteq At$, $$T_{\cal P} (I) = \{hd(r)\, |\, r \in {\cal P}: \body^+(r) \subseteq I, \body^-\cap I =\emptyset\}.$$

\begin{definition}
Let $\cal P$ be a logic program. Define an operator $\Theta_{\cal P}: (2^{At})^2 \rightarrow (2^{At})^2 $ as: 
$\Theta_{\cal P}(T,P) = (\Theta_{\cal P}(T,P)_1,$ $ \Theta_{\cal P}(T,P)_2)$, where
$$
\begin{array}{ll}
\Theta_{\cal P} (T,P)_1 = \{hd(r) \,|\, r \in {\cal P}: \body^+(r) \subseteq T, \body^-(r) \cap P = \emptyset\} 
\\
\Theta_{\cal P} (T,P)_2 = \{hd(r) \,|\, r \in {\cal P}: \body^+(r) \subseteq P, \body^-(r) \cap T = \emptyset\}
\end{array}
$$
\end{definition}

Given a four-valued interpretation $(T,P)$, $\Theta_{\cal P}(T,P)_1$ computes what are true and $\Theta_{\cal P}(T,P)_2$ computes what are possibly true, entirely based on $T$ and $P$ in a uniform manner for  consistent as well as inconsistent interpretations. 
The notion of truth and falsity goes beyond that of classic logic when $(T, P)$ is inconsistent; it is not based on a model-theoretic valuation but on  a set-theoretic evaluation of 
rules. Note that $\Theta_{\cal P}$ is a symmetric operator by definition.


\begin{example}
Let ${\cal P} = \{a \leftarrow \boldnot b., ~ b \leftarrow \boldnot a.\}$.  Applying the definition of stable revision defined in (\ref{stable-revision}), we get  four stable fixpoints of 
$\Theta_{\cal P}$,
$(\emptyset, \{a,b\})$, $\{\{a\}, \{a\}\}$, $\{\{b\}, \{b\}\}$, and $\{\{a,b\}, \emptyset\}$. 
\end{example}

The operator $\Theta_{\cal P}$ above is essentially Fitting's immediate consequence operator  \cite{Fitting02} (with a small adjustment), which is defined as:   given a three-valued interpretation ${\cal I}$,
$$
\begin{array}{ll}
\cal Omega_{\cal P}({\cal I})_1 = \{\head(r) \, |\, r \in {\cal P}: \wholebody(r)^{\cal I} = \bft\}\\
\cal Omega_{\cal P}({\cal I})_2 = \{\head(r) \, |\, r \in {\cal P}: \wholebody(r)^{\cal I} \not = \bff\}
\end{array}
$$
where $\wholebody(r)^{\cal I}$ denotes the standard three-valued valuation based on Kleene's three-valued logic \cite{Kleene38a}.
Notice the subtle difference between operators $\cal Omega_{\cal P}$ and $\Theta_{\cal P}$.  The former is based on classic model-theoretic truth valuation and $\Theta_{\cal P}$ is not. As a result,  while $\Theta_{\cal P}$ is inconsistency-tolerant, $\cal Omega_{\cal P}$ is not. In fact, the $\Psi_{\cal P}$ operator is not well-defined in the modified AFT since $\wholebody(r)^{\cal I}$ is not defined for an inconsistent ${\cal I}$.  

In the literature of AFT it is known that the operator $\Theta_{\cal P}$ is a symmetric  approximator \cite{denecker2000approximations}. The correspondence between three-valued stable models of a logic program and consistent stable fixpoints can be easily derived from the literature (e.g., see
 \cite{Gelder89,DBLP:conf/iclp/Przymusinski90,DBLP:journals/fuin/Przymusinski90}).  Our interest here is that the approximator $\Theta_{\cal P}$ is strong for all consistent stable fixpoints of $\Theta_{\cal P}$.

\begin{proposition}
\label{p}
$\Theta_{\cal P}$ 
is an approximator for $T_{\cal P}$ and it is strong for all consistent stable fixpoints of $\Theta_{\cal P}$.
\end{proposition}
\begin{proof}
Since $\Theta_{\cal P}$ is a symmetric approximator, we only need to verify that it is an approximator for $T_{\cal P}$. This is true because for any $(E,E) \in (2^{At})^2$, by definition we have $\Theta_{\cal P}(E,E)$ equals $(T_{\cal P}(E), T_{\cal P}(E))$.  

To show the property of $\Theta_{\cal P}$ being strong for all consistent stable fixpoints, since $\Theta_{\cal P}$ is a  symmetric approximator, we apply 
Lemma 4.1 of \cite{DeneckerMT04}, which says that for any consistent pair $(x,y)$, if $(x,y)$ is $A^c$-prudent, then $\lfp(A^c (u,\cdot)_2) = \lfp(A (u,\cdot)_2)$. Since any consistent stable fixpoint of $\Theta_{\cal P}$ is
$\Theta_{\cal P}^c$-prudent, we have the desired conclusion.
\end{proof}

\end{appendix}
}

\bibliographystyle{acmtrans}
\bibliography{journal09}

\label{lastpage}
\end{document}